\documentclass[journal]{IEEEtran}

\ifCLASSOPTIONcompsoc
\usepackage[caption=false,font=normalsize,labelfont=sf,textfont=sf]{subfig}
\else
\usepackage[caption=false,font=footnotesize]{subfig}
\fi

\usepackage{graphicx}
\graphicspath{{fig/}}
\usepackage{amsmath}
\usepackage{amssymb}
\usepackage{amsthm}
\usepackage{mathrsfs}
\usepackage{mathtools}
\mathtoolsset{showonlyrefs}
\usepackage{array}
\usepackage{xifthen}
\usepackage{xcolor}
\usepackage{pgfplots}
\pgfplotsset{compat=1.3}
\usetikzlibrary{decorations.markings}
\usetikzlibrary{external}
\usepgfplotslibrary{fillbetween}
\definecolor{clr1}{RGB}{192,0,0}
\definecolor{clr2}{RGB}{0,192,0}
\definecolor{clr3}{RGB}{0,0,192}
\definecolor{clr4}{RGB}{0,192,192}
\definecolor{clr5}{RGB}{192,192,0}
\definecolor{clr6}{RGB}{192,0,192}
\definecolor{clr7}{RGB}{192,192,192}

\newcolumntype{L}[1]{>{\raggedright\let\newline\\\arraybackslash\hspace{0pt}}m{#1}}
\newcolumntype{C}[1]{>{\centering\let\newline\\\arraybackslash\hspace{0pt}}m{#1}}
\newcolumntype{R}[1]{>{\raggedleft\let\newline\\\arraybackslash\hspace{0pt}}m{#1}}

\newtheorem{theorem}{Theorem}

\newtheorem{lemma}[theorem]{Lemma}
\newtheorem{proposition}[theorem]{Proposition}
\newtheorem{definition}[theorem]{Definition}
\newtheorem{remark}[theorem]{Remark}

\newtheorem{example}[theorem]{Example}

\DeclareMathOperator*{\argmin}{arg\!\min}
\newcommand{\st}{\text{s.t.}}
\newcommand{\abs}[1]{\left\lvert#1\right\rvert}
\newcommand{\norm}[1]{\left\lVert#1\right\rVert}
\newcommand{\R}[1]{
	\ifthenelse{\equal{#1}{}}{\mathbb R}{\mathbb R^#1}%
}
\newcommand{\tr}{^T}

\newcommand{\Echg}{E_\text{chg}}
\newcommand{\Emin}{E_\text{min}}

\newcommand{\changed}[1]{{\color{black} #1}}

\begin{document}
	
	\title{Persistification of Robotic Tasks}
	
	\author{Gennaro Notomista and Magnus Egerstedt%
		\thanks{This work was sponsored by the U.S. Office of Naval Research through Grant No. N00014-15-2115.}%
		\thanks{G. Notomista \textit{(corresponding author)} is with the School of Mechanical Engineering, Institute for Robotics and Intelligent Machines, Georgia Institute of Technology, Atlanta, GA 30332 USA (email: g.notomista@gatech.edu)}%
		\thanks{M. Egerstedt is with the School of Electrical and Computer Engineering, Institute for Robotics and Intelligent Machines, Georgia Institute of Technology, Atlanta, GA 30332 USA (email: magnus@gatech.edu)}}%
	
	\maketitle
	
	\begin{abstract}
		In this paper we propose a control framework that enables robots to execute tasks persistently, i.\,e., over time horizons much longer than robots' battery life. This is achieved by ensuring that the energy stored in the batteries of the robots is never depleted. This is framed as a set invariance constraint in an optimization problem whose objective is that of minimizing the difference between the robots' control inputs and nominal control inputs corresponding to the task that is to be executed. We refer to this process as the \textit{persistification} of a robotic task. Forward invariance of subsets of the state space of the robots is turned into a control input constraint by using control barrier functions. The solution of the formulated optimization problem with energy constraints ensures that the robotic task is persistent. To illustrate the operation of the proposed framework, we consider two tasks whose persistent execution is particularly relevant: environment exploration and environment surveillance. We show the persistification of these two tasks both in simulation and on a team of wheeled mobile robots on the Robotarium.
	\end{abstract}
	
	\IEEEpeerreviewmaketitle

	\section{Introduction}
	\label{sec:intro}
	
	\IEEEPARstart{R}{obotic} \changed{tasks such as environmental monitoring and exploration, as well as sensor coverage, typically evolve over long time horizons. However, robots employed to execute these tasks can only store and carry a limited amount of energy in their batteries. For this reason, we can say that such tasks are not \textit{persistent}} insofar as either the robots cannot complete them before their batteries deplete, or they are required to be executed repeatedly and continuously. \changed{Although robots can be designed with greater energy capacity to handle longer-duration tasks, hardware solutions will never allow robots to operate perpetually.}
	
	The objective of this paper is presenting a control framework that provably guarantees the \textit{persistent} execution of robotic tasks. This is achieved by minimally modifying the nominal control inputs corresponding to the task that the robots have to execute in order to ensure the continuous execution of the task. As a result, the robots are allowed to freely execute their task whenever they have enough energy stored in their batteries, whereas they are forced to go and recharge their batteries whenever they are running out of energy.
	
	The deployment of robots for tasks such as environmental monitoring \cite{ollero2005multiple,fiorelli2006multi,leonard2010coordinated,graham2012adaptive}, environment exploration \cite{kuipers1991robot,o2015optimal,burgard2000collaborative} and sensor coverage \cite{cortes2004coverage,lee2015multirobot,krause2008near} has been extensively investigated. However, despite the fact that, in most cases, these tasks have to be executed over long time horizons, the limited availability of energy is not directly taken into account. \changed{Nevertheless, since low energy density is still a severe limiting factor in many mobile robotic applications, energy-awareness is a necessary feature with which robots have to be endowed \cite{morrissurvey}}. This line of inquiry has been followed in \cite{mitchell2015multi}, which considers a multi-robot, persistent coverage problem as a variant of the vehicle routing problem. A heuristic algorithm is proposed that is based on the cost-to-go-to-target, which can be adjusted online to take into account detours that pass through refueling stations present in the environment. A different approach is adopted in \cite{bethke2009multi}, where a formulation based on Markov decision processes is presented, that is able to ensure persistence surveillance coverage, including communication constraints and sensor failure models.
	
	Energy-aware control policies for persistent surveillance using a team of robots are considered in \cite{derenick2011energy}, where an optimization problem is defined in order to trade-off between the coverage mission and the conservation of a desired energy level. This is achieved by transitioning between coverage-directed and charging-directed behaviors, based on the current energy levels. The coverage performances are improved by employing ``standby'' robots that can be deployed when a robot is docked at a charging station. Similarly, in \cite{kamra2015mixed} a solution to the problem of long-duration missions is proposed, which considers the team of robots split into ``task robots'', which are in charge of executing tasks, and ``delivery robots'', which are deployed with the goal of providing the task robots with the energy resources they request. Also in \cite{mathew2015multirobot}, the strategy consists in making use of a team of robots dedicated to charging tasks (ground mobile docking stations), whose trajectories are planned based on the working robots' trajectories (UAVs), in order to guarantee randez-vous and recharge without suspending the operation of the working robots. Limited energy reserve is used as an additional constraint in \cite{liu2014energy} for a path planning strategy for the optimal deployment of multi-robot teams. A definition of persistency different from the one introduced in this paper is considered in \cite{smith2012persistent}, where environment persistent monitoring is solved by varying the robots' speed along predefined trajectories with the goal of keeping a regenerating environment information field bounded, analogously to what happens to a robot cleaner in an environment in which dust is generated. 
	
	The \textit{persistification} approach, through which a robotic task is rendered persistent, that we present in this paper, leverages control barrier functions (CBFs) to formulate an optimization problem where the task execution is constrained by the robots' energy level. \changed{As formalized in \cite{egerstedt2018robot} and \cite{notomista2019constraint}, the \textit{constraint-driven} formulation resulting from the use of CBFs is particularly amenable for long-duration robotic tasks, where \textit{goal-driven} strategies, derived starting from precise model assumptions, do not guarantee enough robustness.}
	
	CBFs can be used to synthesize a controller that ensures the forward invariance of a set $\mathcal C$ of the robot state space. This way, defining $\mathcal C$ as the set where the battery energy level of the robots executing the task is always greater than a desired minimum value, the persistification of a task can be formally guaranteed by ensuring the forward invariance of $\mathcal C$. The first definition of CBFs was given in \cite{wieland2007constructive}; in this paper, we are using the one introduced in \cite{ames2014control}. See \cite{ames2019control} for a survey on the subject. In \cite{ames2016control,nguyen2016exponential,wu2016safety}, variations have been introduced in order to employ CBFs with different categories of nonlinear dynamic systems for different control applications.
	
	The persistification framework described in the following extends the work in \cite{notomista2018persistification}. More specifically, the main contribution of this paper is the reformulation of the task persistification introduced in \cite{notomista2018persistification} in order to be able to handle more complicated robot and energy dynamics. The presented method generalizes to different charging models as well as different robotic tasks. The tasks will be encoded through different nominal inputs to the robots. This allows us to formally guarantee the persistent execution of a large number and variety of robotic tasks.
	
	The remainder of the paper is organized as follows: in Section~\ref{sec:problem}, the environment and robot models considered in this paper are presented, and the problem of persistification of robotic tasks is formulated. In Section~\ref{sec:control}, the control framework required to ensure robotic task persistence is introduced and its application is demonstrated by means of preliminary examples. Section~\ref{sec:applications} discusses the application of the presented framework to the persistification of robotic tasks. Section~\ref{sec:experiments} is dedicated to the validation of the proposed theoretical formulation through the implementation of two robotic tasks whose persistent execution is particularly relevant, both in simulation and on a team of mobile robots.
	
	\section{Problem Formulation}
	\label{sec:problem}
	
	The goal of this paper is the persistification of robotic tasks, i.\,e., ensuring that the battery energy level of the robots executing the tasks never falls below a minimum value. In this section, we introduce the models used for the robots, the environment in which the robots execute their task, and the robots' energy dynamics. We conclude the section by formalizing the persistification of robotic tasks.
	
	\subsection{Robot Model}
	
	Consider a collection of $N$ robots which are to be deployed to execute a task. The state of robot $i$, $i = 1,\ldots,N$, is denoted by $x_i \in X \subseteq \R{n}$. We model the robots using the control affine dynamical system:
	\begin{equation}
		\label{eq:robotmodelctrlaff}
		\dot x_i = f(x_i) + g(x_i) u_i,
	\end{equation}
	where $u_i \in U \subseteq \R{m}$ is robot $i$'s input, and $f$ and $g$ are two Lipschitz vector fields. Control affine dynamics arise in many robotic systems, whose models are derived using Euler-Lagrange equations (as observed in \cite{lavalle2006planning}), therefore they lend themselves to the description of a large variety of robotic platforms. Throughout this paper, we will assume that the $N$ robots are homogeneous, i.\,e., $f$ and $g$ in \eqref{eq:robotmodelctrlaff} are the same for each robot. This assumption does not compromise the proposed persistification strategy, which can be easily employed in the heterogeneous case, as will be explained in Section~\ref{sec:applications}.
	
	The robots considered in this paper are equipped with a rechargeable source of energy, e.\,g., a battery, and a technology required to recharge it, e.\,g., solar panels. In the following two subsections, we present a model for the robot energy dynamics which is coupled with the model of the environment in which the robots move. This encompasses, for instance, the situation in which solar power harvester circuits are employed to use solar panels to recharge the battery of the robots.
	
	\subsection{Environment Model}
	\label{subsec:environment}
	
	The environment, i.\,e., the domain in which the robots are deployed to perform their task, is represented by the compact set $\mathcal E \subset \R{p}$, with $p=2$ or $p=3$, for ground or aerial robots, respectively. The function
	\begin{equation}
		\pi : X \subseteq \R{n} \to \mathcal E \subset \R{p}
		\label{eq:maptostate}
	\end{equation}
	maps the robot state to its position expressed in a Cartesian reference system defined in the environment $\mathcal E$.
	
	Moreover, we consider the time-varying scalar field
	\begin{equation}
		I : \mathcal E \times \R{}_+ \to \mathcal I \subset \R{}_+,
		\label{eq:solarintensity}
	\end{equation}
	where $\mathcal I$ is an interval, defined over the environment, which represents a bounded time-varying non-negative physical quantity (e.\,g., solar light intensity) associated to each position in $\mathcal E$ at each time instant. We insist on $I$ being Lipschitz continuous in its first argument and differentiable in its second argument. The need for these assumptions will be explained in the next section.
	
	\subsection{Energy Model}
	\label{subsec:energy}
	
	Let $E_i \in \R{}_+$ be the battery energy level of robot $i$. The charging and discharging dynamics of the battery are modeled by:
	\begin{equation}
		\dot E_i = F(x_i,E_i,t) = k\left(w(x_i,E_i,t) - E_i\right),
		\label{eq:energymodel}
	\end{equation}
	where $k>0$ and
	\begin{equation}
		\label{eq:wxet}
		w(x_i,E_i,t) = \dfrac{1}{1+\frac{1-E_i}{E_i}e^{-\lambda(I(x_i,t)-I_\text{c})}}.
	\end{equation}
	\changed{In \eqref{eq:wxet}, $\lambda > 0$ and $0 < I_\text{c} < 1$ are two scalars whose meaning and effect on the energy dynamics will be explained in Remark~\ref{rmk:solarintensity}, and $I(x_i,t) : \mathcal E \times \R{}_+ \to [0,1]$ is a time-varying scalar-valued function introduced in \eqref{eq:solarintensity}.}
	
	The specific expression of $F(x_i,E_i,t)$ has been defined in order to model the exponential charging-discharging dynamics of batteries that are used to power robotic platforms in a large number of applications \cite{daniel2012handbook}.
	\begin{figure*}
		\centering
		\subfloat[Data collected during the course of a 24-hour experiment using a solar-powered robot. $E$, in blue, and $I$, in red, are measured battery energy and solar light intensity, respectively.]{
			\label{subfig:ei}
			\begin{tikzpicture} % scale and rotate
			%	\Large % axis font size
			\begin{axis}
			[
			no marks, % remove marks
			xlabel={data points}, % xlabel,
			xtick={8,16,24}, % xtick
			xticklabels={480,960,1440}, % xticklabel
			ylabel={\textcolor{blue}{$E$},\textcolor{red}{$I$}}, % ylabel
			%		ytick={0,0.52,0.9}, % ytick
			%		yticklabels={0,$\Emin$, $\Echg$}, % yticklabel
			%		xmin=-5, % xlim
			%		xmax=5, % xlim
			ymin=0, % ylim
			ymax=1, % ylim
			enlarge x limits=-1, % x axis tight
			enlarge y limits=-1, % y axis tight
			%		axis equal image, % axis equal
			%		grid=both, % grid on
			%		xmajorgrids, % x grid on
			%		ymajorgrids, % y grid on
			%		minor tick num=2, % grid minor
			%		legend entries={one-column plot,test12,test34,test56}, % legend
			%		legend style={nodes=right}, % legend style
			%		legend pos= north west, % legend position
			width=0.31\textwidth, % image width
			height=0.25\textwidth % image height
			]
			\addplot [line width=2pt, color=blue] table [x=t, y=V, col sep=space]{data/energy_intensity.txt};
			\addplot [line width=2pt, color=red] table [x=t, y=I, col sep=space]{data/energy_intensity.txt};
			\end{axis}
			\end{tikzpicture}
		}\hfill
		\subfloat[\changed{Comparison between measured and predicted values of $\dot E$: $\Delta\dot E$ represents the difference between the measured $\dot E$ and its value predicted using the model in \eqref{eq:energymodel}. The mean of $\Delta\dot E$ is depicted as a thick green line, whereas the shaded area represents the region of one standard deviation from the mean value.}]{
			\label{subfig:iedot}
			%\begin{tikzpicture} % scale and rotate
			%%	\Large % axis font size
			%	\begin{axis}
			%	[
			%		no marks, % remove marks
			%		xlabel={$I$}, % xlabel,
			%%		xtick={8,16,24}, % xtick
			%%		xticklabels={480,960,1440}, % xticklabel
			%		ylabel={$\dot E$}, % ylabel
			%%		ytick={0,0.52,0.9}, % ytick
			%%		yticklabels={0,$\Emin$, $\Echg$}, % yticklabel
			%%		xmin=-5, % xlim
			%%		xmax=5, % xlim
			%%		ymin=0, % ylim
			%%		ymax=1, % ylim
			%		enlarge x limits=-1, % x axis tight
			%		enlarge y limits=-1, % y axis tight
			%%		axis equal image, % axis equal
			%%		grid=both, % grid on
			%%		xmajorgrids, % x grid on
			%%		ymajorgrids, % y grid on
			%%		minor tick num=2, % grid minor
			%%		legend entries={one-column plot,test12,test34,test56}, % legend
			%%		legend style={nodes=right}, % legend style
			%%		legend pos= north west, % legend position
			%		width=0.32\textwidth, % image width
			%		height=0.25\textwidth % image height
			%	]
			%		\draw [black, dashed, line width=2pt] (axis cs: 0,-0.15) rectangle (axis cs: 0.86,0.25);
			%		\addplot [color=black!50!green, only marks, mark size=1pt] table [x=I, y=VdotMeas, col sep=space]{data/model_plot.txt};
			%		\addplot [color=green, only marks, mark size=1pt] table [x=I, y=VdotModel, col sep=space]{data/model_plot.txt};
			%	\end{axis}
			%\end{tikzpicture}
			\begin{tikzpicture} % scale and rotate
			%	\Large % axis font size
			\begin{axis}
			[
			no marks, % remove marks
			xlabel={$E$}, % xlabel,
			%		xtick={8,16,24}, % xtick
			%		xticklabels={480,960,1440}, % xticklabel
			ylabel={$\Delta \dot E$}, % ylabel
			%		ytick={0,0.52,0.9}, % ytick
			%		yticklabels={0,$\Emin$, $\Echg$}, % yticklabel
			xmin=0, % xlim
			xmax=1, % xlim
			ymin=-1, % ylim
			ymax=1, % ylim
			enlarge x limits=-1, % x axis tight
			enlarge y limits=-1, % y axis tight
			%		axis equal image, % axis equal
			%		grid=both, % grid on
			%		xmajorgrids, % x grid on
			ymajorgrids, % y grid on
			%		minor tick num=2, % grid minor
			%		legend entries={one-column plot,test12,test34,test56}, % legend
			%		legend style={nodes=right}, % legend style
			%		legend pos= north west, % legend position
			width=0.31\textwidth, % image width
			height=0.25\textwidth % image height
			]
			\addplot[line width=1pt, color=green!70!black] table [x=V, y=m, col sep=space]{data/model_meanstd.txt};
			\addplot[name path=ssup,color=green!70] table [x=V, y=ssup, col sep=space]{data/model_meanstd.txt};
			\addplot[name path=sinf,color=green!70] table [x=V, y=sinf, col sep=space]{data/model_meanstd.txt};
			\addplot[green!50,fill opacity=0.5] fill between[of=ssup and sinf];
			\end{axis}
			\end{tikzpicture}
		}\hfill
		\subfloat[Simulated battery charging and discharging dynamics: a comparison with Fig.~\protect\ref{subfig:ei} shows that the model proposed in \eqref{eq:energymodel} is able to reproduce the true dynamics of a real battery used in robotics applications.]{
			\label{subfig:cd}
			\begin{tikzpicture} % scale and rotate
			%	\Large % axis font size
			\begin{axis}
			[
			no marks, % remove marks
			xlabel={$t$}, % xlabel,
			%		xtick={8,16,24}, % xtick
			%		xticklabels={480,960,1440}, % xticklabel
			ylabel={$E$}, % ylabel
			%		ytick={0,0.52,0.9}, % ytick
			%		yticklabels={0,$\Emin$, $\Echg$}, % yticklabel
			%		xmin=-5, % xlim
			%		xmax=5, % xlim
			%		ymin=0, % ylim
			%		ymax=1, % ylim
			enlarge x limits=-1, % x axis tight
			enlarge y limits=-1, % y axis tight
			%		axis equal image, % axis equal
			%		grid=both, % grid on
			%		xmajorgrids, % x grid on
			%		ymajorgrids, % y grid on
			%		minor tick num=2, % grid minor
			legend entries={$I\ge I_\text{c}$,$I<I_\text{c}$}, % legend
			legend style={nodes=right}, % legend style
			legend pos= north east, % legend position
			width=0.31\textwidth, % image width
			height=0.25\textwidth % image height
			]
			\addplot [line width=2pt,color=black] table [x=t, y=Vc, col sep=space]{data/charge_discharge.txt};
			\addplot [dashed, line width=2pt,color=black] table [x=t, y=Vd, col sep=space]{data/charge_discharge.txt};
			\end{axis}
			\end{tikzpicture}
		}
		\caption{Validation of the energy model proposed in \eqref{eq:energymodel} using data collected during a long-term experiment with a solar-powered robot \cite{notomista2019slothbot}.}
		\label{fig:energymodel}
	\end{figure*}
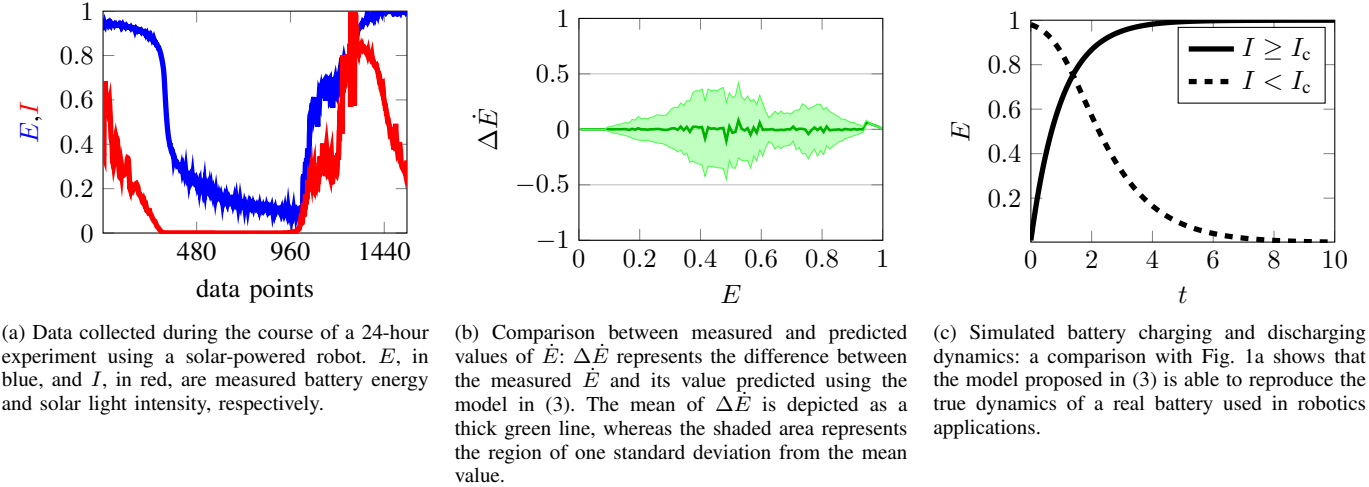
	In our previous work \cite{notomista2019slothbot}, the design of a solar-powered robot is presented and data of solar intensity and battery charge collected during the course of a 1-day-long experiment are reported (see Fig.~\ref{subfig:ei}). \changed{Figure~\ref{subfig:iedot} shows the difference $\Delta\dot E$ between the measured $\dot E$ and its value predicted using the energy model \eqref{eq:energymodel}. The thick green line depicts the mean of $\Delta\dot E$ and the shaded area denotes the region of one standard deviation around the mean.} Moreover, Fig.~\ref{subfig:cd} shows the simulated battery charge and discharge curves that can be obtained using \eqref{eq:energymodel} in the cases when $I>I_\text{c}$ and $I<I_\text{c}$, respectively: a comparison with Fig.~\ref{subfig:ei} indicates that the theoretical model is able to capture the exponential charging and discharging dynamics of real batteries commonly used for robotics applications.
	
	\begin{figure}[th]
		\centering
		\includegraphics[width=.49\textwidth]{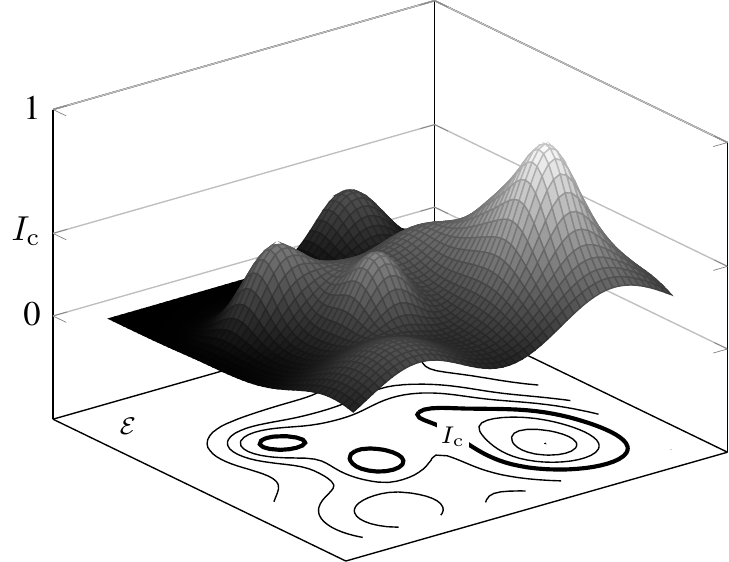}
		\caption{Example of the function $I(x_i,t)$ over the environment $\mathcal E$ at a given time instant: inside the bold level curve marked with $I_\text{c}$, $\dot E_i > 0$, whereas outside $\dot E_i < 0$. In other words, the robots can recharge their batteries in the regions bounded by the bold curves.}
		\label{fig:ixei}
	\end{figure}
	\begin{remark}
		\label{rmk:solarintensity}
		$I(x_i,t)$ can be interpreted as a time-varying power source distributed over the environment $\mathcal E$. For instance, in the case where it represents a measure of the solar light intensity at the position $x_i$ and time $t$, $F(x_i,E_i,t)$ can be used to describe the energy dynamics of a solar rechargeable battery.
		With the given dynamics we have that:
		\begin{itemize}
			\item $\dot E_i < 0$, i.\,e., the battery is discharging, whenever $I(x_i,t) < I_\text{c}$
			\item $\dot E_i > 0$, i.\,e., the battery is charging, when $I(x_i,t) > I_\text{c}$
			\item $\dot E_i = 0$, i.\,e., the value of $I(x_i,t) = I_\text{c}$ is such that the generated energy is equal to the energy required by the robot at time $t$.
		\end{itemize}
		Figure~\ref{fig:ixei} shows an example of what has been described: the surface plot of the function $I$ at a given time instant $t$ is depicted in grayscale (black to white for values of $I$ that go from 0 to 1). Below the surface plot, the contour plot of $I$ highlights the level curves where $I(x_i,t)=I_\text{c}$. Inside the regions bounded by the bold curves, characterized by $I(x_i,t)>I_\text{c}$, $\dot E_i > 0$, i.\,e., the robots can charge their batteries.
	\end{remark}
	\begin{figure}[th]
		\centering
		\subfloat[]{\label{subfig:cs1}\includegraphics[width=.32\textwidth]{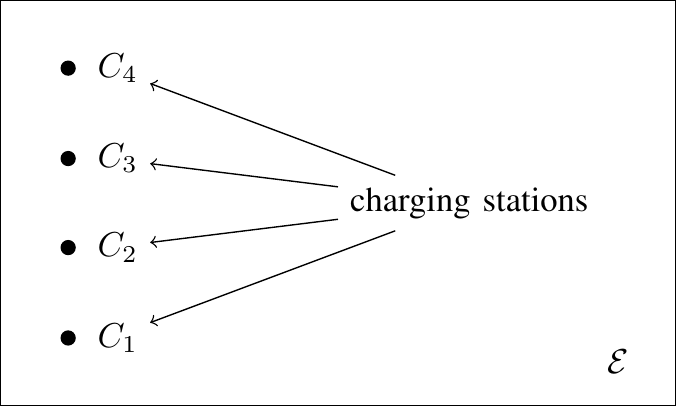}}\\
		\subfloat[]{\label{subfig:cs2}\includegraphics[width=.49\textwidth]{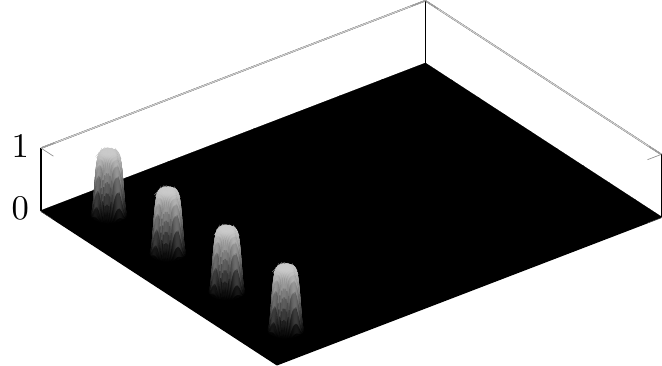}}
		\caption{Example of the modeling of lumped sources of energy (charging stations) via a suitable $I(x_i,t)$ function. In Fig.~\protect\ref{subfig:cs1}, a rectangular environment $\mathcal E$ is shown, and the positions of four charging stations, denoted by $C_1$ to $C_4$, are depicted as black dots. Fig.~\protect\ref{subfig:cs2} shows the surf plot of $I(x_i,t)$ corresponding to the charging stations of Fig.~\protect\ref{subfig:cs1} modeled by means of bump-like functions.}
		\label{fig:lumpedenergysources}
	\end{figure}
	\begin{remark}
		Lumped sources of energy, such as charging stations, can be also modeled using \eqref{eq:solarintensity}. Bump-like functions \cite{tu2008bump} at the locations of the charging stations can be employed to obtain the desired charging behavior, as depicted in Fig.~\ref{fig:lumpedenergysources}.
	\end{remark}
	
	\begin{remark}
		\label{rmk:energyinput}
		The proposed energy model does not depend on the control input $u_i$ of the robot. At first, this can seem inaccurate, however in Section~\ref{sec:applications} we will show why this choice was made and how it increases the robustness of the proposed persistification approach, ensuring that the robots will never run out of energy while executing the given task.
	\end{remark}
	
	\subsection{Task Persistification}
	
	The compound model of robot and energy dynamics is given by the following set of differential equations:
	\begin{equation}
		\setlength{\nulldelimiterspace}{0pt}
		\left\{\begin{IEEEeqnarraybox}[\relax][c]{l's}
			\dot x_i = f(x_i) + g(x_i) u_i\\
			\dot E_i = F(x_i,E_i,t).
		\end{IEEEeqnarraybox}\right.
		\label{eq:robotmodel}
	\end{equation}
	\changed{Indicating the augmented state of robot $i$ by
		\[
		z_i = \begin{bmatrix}x_i\\E_i\end{bmatrix},
		\]
	}the robot model \eqref{eq:robotmodel} can be rewritten in the following control affine form:
	\begin{equation}
		\dot z_i = \hat f(z_i,t) + \hat g(z_i) u_i,
		\label{eq:robotmodelaffine}
	\end{equation}
	where
	\begin{equation}
		\hat f(z_i,t) = \begin{bmatrix}
			f(x_i)\\
			F(x_i,E_i,t)
		\end{bmatrix}
		\quad\text{and}\quad
		\hat g(z_i) = \begin{bmatrix}
			g(x_i)\\
			0
		\end{bmatrix}.
	\end{equation}
	We will use
	\begin{equation}
		x = \begin{bmatrix}
			x_1\\
			\vdots\\
			x_N
		\end{bmatrix}\in X^{Nn},\quad
		u = \begin{bmatrix}
			u_1\\
			\vdots\\
			u_N
		\end{bmatrix}\in U^{Nm}
	\end{equation}
	to represent the joint states and inputs of the $N$ robots performing the task to be persistified.
	
	Let the task that has to be executed by the robots be encoded through the nominal input
	\begin{equation}
		\hat u_i(x,t),\quad i=1,\ldots,N,
	\end{equation}
	or collectively as:
	\begin{equation}
		\hat u : X^{Nn} \times \R{}_+ \to U^{Nm}.
		\label{eq:nominput}
	\end{equation}
	\changed{This modeling choice for the tasks that the robots have to execute is general insofar as it encompasses both reactive feedback controllers, through the dependence of $\hat u$ from the state $x$, and controllers generated by a high-level planning strategy, through the dependence of $\hat u$ from the time $t$.
		
		Examples of such tasks include, for instance, the stabilization of a dynamical system, the coordinated control of multi-robot systems, as well as robot collision avoidance strategies. The execution of robotic tasks using this formulation will be showcased in Section~\ref{sec:applications}.}
	
	\begin{definition}[Task Persistification]
		A task encoded through the nominal input $\hat u$, defined in \eqref{eq:nominput}, is persistified if the robots employed to perform it execute the input $u^\ast$, solution of the following program:
		\begin{equation}
			\begin{aligned}
				u^\ast(x,t) = \argmin_{u}  & \norm{u-\hat u(x,t)}^2\\
				\st \enskip & E_i(t) \in [\Emin,~\Echg]\quad\forall i\in \{1,\ldots, N\}
			\end{aligned}
			\label{eq:persistification}
		\end{equation}
		$\forall t \ge 0$, where $E_i$ is the energy of robot $i$ introduced in Section~\ref{subsec:energy}, and $\Emin>0$ and $\Echg$, with $1\ge\Echg>\Emin>0$, are the minimum and the maximum energy values, respectively, between which the energy of the robots should be confined.
		
		Task persistification is referred to as the process of turning a task characterized by the nominal input $\hat u$ to the persistified task characterized by the input $u^\ast$.
		
		\label{goal:gpt}
	\end{definition}
	
	\begin{remark}
		Definition~\ref{goal:gpt} is not tailored to the specific robot and environment models, which, as a matter of fact, can be quite different from the ones presented above, depending on the particular application that is considered.
	\end{remark}
	
	\section{Control Framework}
	\label{sec:control}
	
	As explained in Section~\ref{sec:intro}, the task persistification expressed as the optimization problem \eqref{eq:persistification}, will be realized by employing control barrier functions (CBFs). In the following subsection, we give a brief introduction to CBFs in their basic form. Then, in Section~\ref{subsec:maincontribution}, we present the extensions that represent the main theoretical contribution of this paper, namely a flexible framework to deal with time-varying and high relative degree CBFs and control Lyapunov functions (CLFs). The results obtained in this section will be employed in Section~\ref{sec:applications} to formulate the optimization program whose solution achieves the persistification of robotic tasks.
	
	\subsection{Control Barrier Functions}
	
	Control barrier functions have been used with the goal of ensuring \textit{safety}, intended as the invariance property of a subset of the state space, the \textit{safe set}. In the following we give the definition introduced in \cite{xu2015robustness} which will be extended in this paper.
	
	\begin{definition}[\cite{xu2015robustness}]
		\label{def:zcbf}
		Let $h:X\subset\R{n}\to\R{}$ be a continuously differentiable function, and $\mathcal C$ its zero superlevel set, i.\,e., $\mathcal C = \{x\in\R{n}~|~h(x)\ge0\}$. Then, for a control affine system
		\begin{equation}
			\label{eq:controlaffine}
			\dot x = f(x) + g(x) u,
		\end{equation}
		$x\in X\subset\R{n}$, $u\in U\subset\R{m}$, $h$ is a control barrier function (CBF) if there exists a locally Lipschitz extended class $\mathcal K$ function \cite{xu2015robustness} $\alpha$ such that
		\begin{equation}
			\sup_{u\in U} \left[ L_f h(x) + L_g h(x) u +\alpha(h(x))\right] \ge0,
		\end{equation}
		for all $x$ in the interior of the set $\mathcal C$. $L_f h(x)$ and $L_g h(x)$ represent the Lie derivative of $h$ in the directions of the vector fields $f$ and $g$, respectively.
	\end{definition}
	Starting from this basic definition, in the next sections we will develop tools required for the persistification of robotic tasks.
	
	\subsection{High Relative Degree CBFs and CLFs}
	\label{subsec:maincontribution}
	
	In the previous section, we introduced the robot and energy models, and in Remark~\ref{rmk:energyinput}, we have pointed out that the energy dynamics do not depend explicitly on the robot input $u_i$. This is intended to be a conservative choice which increases the robustness of the persistification strategy. In fact, the rate of charge and discharge of the battery, obtained when $I(x_i,t)=1$ and $I(x_i,t)=0$, respectively, are designed to be the rates obtained when the robot input $u_i$ attains its maximum value. This would correspond, for instance, to the fastest discharge rate obtained when the actuators of the robot are absorbing maximum power. Then, the actual discharge rate will always be slower than the modeled one, increasing, this way, the robustness of the proposed persistification strategy against unmodeled phenomena which can occur in the environment, or unmodeled robot dynamics. Nevertheless, the gained robustness comes at the price of increasing the relative degree of the CBF $h$, defined below.
	\begin{definition}[Relative degree of a CBF, based on \cite{nguyen2016exponential}]
		Given the nonlinear system \eqref{eq:controlaffine}, with $f$ and $g$ sufficiently smooth vector fields on a domain $\mathcal D$, the CBF $h : \R{n} \to \R{}$ has relative degree $\rho$, $1\le\rho\le n$, in $\mathcal D_0 \subset \mathcal D$ if the system
		\begin{equation}
			\begin{cases}
				\dot x = f(x) + g(x) u\\
				y = h(x)
			\end{cases}
		\end{equation}
		has a relative degree $\rho$, $\forall x \in \mathcal D_0$.
	\end{definition}
	In the following, we give an example in which high relative degree CBFs \cite{nguyen2016exponential} are required. This example will be generalized in Theorem~\ref{thm:nestedcbfs} for arbitrarily high relative-degree CBFs. Theorem~\ref{thm:nestedcbfs} will be then applied in order to formulate an optimization program which realizes the proposed persistification strategy.
	
	\begin{example}[Cascade of CBFs]
		Let us consider the nonlinear dynamical system in control affine form \eqref{eq:controlaffine} and a sufficiently smooth function $h_1 : \R{n} \to \R{}$ with relative degree 2 (i.\,e., $L_gh_1(x) = 0$ and $L_gL_fh_1(x) \neq 0$) that defines the superlevel set $\mathcal C_1 = \left\{ x \in \R{n} ~\vert~ h_1(x)\ge0 \right\}$. To prove the forward invariance of the set $\mathcal C_1$,
		we want $h_1$ to be a CBF for which the following must hold:
		\begin{equation}
			L_f h_1(x) + \alpha_1(h_1(x)) \ge 0,
		\end{equation}
		where $\alpha_1$ is a continuously differentiable extended class $\mathcal K$ function, and we used the fact that $h_1$ has relative degree 2.
		Then, we can define an additional function
		\begin{equation}
			h_2(x) = L_f h_1(x) + \alpha_1(h_1(x))
			\label{eq:additionalcbf}
		\end{equation}
		whose zero superlevel set is $\mathcal C_2 = \left\{ x \in \R{n} ~\vert~ h_2(x)\ge0 \right\}$. If there exists a positive constant $\gamma_1$ and a locally Lipschitz extended class $\mathcal K$ function $\alpha_2$ such that
		\begin{equation}
			\begin{aligned}
				\sup_{u \in  U} \big[ &L_f^2 h_1(x) + L_g L_f h_1(x)\,u \\
				&+ \gamma_1L_f h_1(x) + \alpha_2(h_2(x))\big] \ge 0,
			\end{aligned}
			\label{eq:cascade2}
		\end{equation}
		then the function $h_2$ is a CBF. The condition in \eqref{eq:cascade2} has been obtained using the definition \eqref{eq:additionalcbf} and employing $\alpha_1(s)=\gamma_1 s$. Its expression for an arbitrary high relative degree will be given in the next theorem. The existence of the CBF $h_2(x)$ ensures the forward invariance of the set $\mathcal C_2$, which, in turn, ensures the existence of the CBF $h_1(x)$. The forward invariance of the set $\mathcal C_1$ is thus proved.
		\label{ex:cascadecbf}
	\end{example}
	
	The technique shown in example~\ref{ex:cascadecbf} is generalized in the following theorem.
	\begin{theorem}
		\label{thm:nestedcbfs}
		Given a dynamical system \eqref{eq:controlaffine}, a sufficiently smooth CBF $h_1(x)$ with relative degree $\rho$ and a CBF $h_\rho(x)$ \changed{which can be evaluated recursively starting from $h_1(x)$ using the following equation}
		\begin{equation}
			h_{n+1}(x) = \dot h_n(x) + \alpha_n(h_n(x)),\quad\changed{1\le n<\rho},
			\label{eq:recursivecbf}
		\end{equation}
		with $\alpha_n$ continuously differentiable extended class $\mathcal K$ functions, we define the set $K_\rho(x)$ as
		\begin{equation}
			\begin{aligned}
				K_\rho(x) &= \biggl\{ u \in U ~\biggl\vert~ L_f^\rho h_1(x) + L_gL_f^{\rho-1}h_1(x) u\\
				&+ \sum_{i=1}^{\rho-1}\sum_{\mathscr C \in \binom{\rho-1}{i}}\prod_{j\in \mathscr C}\frac{\partial \alpha_j}{\partial h_j}L_f^{\rho-i}h_1(x) + \alpha_\rho(h_\rho(x)) \ge 0 \biggr.\biggr\},
			\end{aligned}
		\end{equation}
		where $\binom{\rho-1}{i}$ is the set of $i$-combinations from the set $\{1, \ldots, \rho-1\} \subset \mathbb N$ and $\alpha_\rho$ is a locally Lipschitz extended class $\mathcal K$ functions. Then, any Lipschitz continuous controller $u \in K_\rho(x)$ will render the set $\mathcal C_1 = \left\{ x \in \R{n} ~\vert~ h_1(x)\ge0 \right\}$ forward invariant.
	\end{theorem}
	\begin{proof}
		Given the properties of class $\mathcal K$ functions \cite{kellett2014compendium} and using the fact that $h_1(x)$ has relative degree $\rho$, $\dot h_\rho(x)$ is given by the following expression:
		\begin{equation}
			\begin{aligned}
				\dot h_\rho &= L_f^\rho h_1(x) + L_gL_f^{\rho-1}h_1(x) u \\
				&+ \sum_{i=1}^{\rho-1}\sum_{\mathscr C \in \binom{\rho-1}{i}}\prod_{j\in \mathscr C}\frac{\partial \alpha_j}{\partial h_j}L_f^{\rho-i}h_1(x).
			\end{aligned}
		\end{equation}
		Consequently, the choice of $u \in K_\rho(x)$ renders the set $\mathcal C_\rho = \left\{ x \in \R{n} ~\vert~ h_\rho(x)\ge0 \right\}$ forward invariant. By recursively applying \eqref{eq:recursivecbf} $\rho-1$ times, $\mathcal C_1$ is proved to be forward invariant.
	\end{proof}
	
	\begin{remark}
		Employing a cascade of CBFs as shown in Example~\ref{ex:cascadecbf} and in Theorem~\ref{thm:nestedcbfs} is a technique which can be used not only to prove set forward invariance, but also stability of dynamical systems using high relative degree Lyapunov functions, as will be shown in the following. Set forward invariance and stability will be used, in Section~\ref{sec:applications}, to ensure that the energy stored in the battery of the robots performing a task is never depleted, realizing, this way, the desired task persistification.
	\end{remark}
	
	Note that the energy model proposed in Section~\ref{subsec:energy} is time-dependent, as it depends on the environment model, through the value $I(x_i,t)$ at $x_i$ at time $t$. As we will use CBFs to ensure the persistent execution of a task, we now extend the notion of CBFs to the case in which the function $h$ that defines the safe set $\mathcal C$ explicitly depends on time.
	
	For the nonlinear control affine system \eqref{eq:controlaffine}, we wish to ensure the forward invariance of a set $\mathcal C \subset \R{n}$ defined by the superlevel set of a function $h : \R{n} \times \R{}_+ \to \R{}$ as:
	\begin{equation}
		\mathcal C = \left\{ x \in \R{n} ~\vert~ h(x,t)\ge0 \right\}.
		\label{eq:safeset}
	\end{equation}
	With this objective, we extend the definition of CBFs given in \cite{xu2015robustness} to the time-varying case.
	\begin{definition}[Time-Varying CBFs]
		\label{def:tvzcbf}
		Given a dynamical system \eqref{eq:controlaffine} and a set $\mathcal C$ defined in \eqref{eq:safeset}, the function $h$ is a time-varying CBF defined on $\mathcal D \times \R{}_+$, with $\mathcal C \subseteq \mathcal D \subset \R{n}$, if there exists a locally Lipschitz extended class $\mathcal K$ function $\alpha$ such that, $\forall x \in \mathcal D$,
		\begin{equation}
			\sup_{u \in  U} \left[ \frac{\partial h}{\partial t} + L_f h(x,t) + L_g h(x,t)\,u + \alpha(h(x,t))\right] \ge 0.
			\label{eq:tvfcbf}
		\end{equation}
		
	\end{definition}
	\noindent Starting from the condition in \eqref{eq:tvfcbf}, we can define the set:
	\begin{equation}
		\begin{aligned}
			K(x,t) = \Bigg\{ u \in  U \,\Bigg\vert\, &\frac{\partial h}{\partial t}+ L_f h(x,t) + L_g h(x,t)\,u \\
			& + \alpha(h(x,t))\ge0 \Bigg\}.
		\end{aligned}
		\label{eq:inputset}
	\end{equation}
	
	The following lemma ensures that the set $\mathcal C$, defined in \eqref{eq:safeset}, is rendered forward invariant by the application of a control input $u\in K(x,t)$. This result will be used in Section~\ref{sec:applications} to express the constraints on the energy in \eqref{eq:persistification} in terms of the control input $u$.
	
	\begin{lemma}
		\label{thm:tvzcbf}
		Given a set $\mathcal C$ defined as in \eqref{eq:safeset}, if $h$ is a time-varying CBF on $\mathcal D \times \R{}_+$, then any Lipschitz continuous controller $u\in K(x,t)$, where $K(x,t)$ is given in \eqref{eq:inputset}, will render the set $\mathcal C$ forward invariant.
	\end{lemma}
	\begin{proof}\textit{(Following the proof of Theorem 1 in \cite{ames2014control})}
		If $h$ is a time-varying CBF on $\mathcal D \times \R{}_+$ and $u \in K(x,t)$, from \eqref{eq:tvfcbf} and \eqref{eq:inputset} we can derive the following differential inequality:
		\begin{equation}
			\dot h(x,t)  \ge -\alpha(h(x,t)).
			\label{eq:diffineq}
		\end{equation}
		Now, consider the following boundary condition problem:
		\begin{equation}
			\begin{cases}
				\dot\zeta = -\alpha(\zeta)\\
				\zeta(t_0) = h(x(t_0),t_0) > 0,
			\end{cases}
		\end{equation}
		whose solution is given by $\zeta(t)=\beta(\zeta(t_0),t-t_0)$, $\beta$ being a class $\mathcal{KL}$ function (Lemma 4.4 in \cite{khalil2002nonlinear}). From \eqref{eq:diffineq}, making use of the Comparison Lemma (Lemma 3.4 in \cite{khalil2002nonlinear}), we have that:
		\begin{equation}
			h(x(t),t) \ge \beta(\zeta(t_0),t-t_0),~\;\forall t\ge t_0.
			\label{eq:hzt}
		\end{equation}
		Hence, if $h(x(t_0),t_0)>0$ and therefore $x(t_0) \in \mathcal C$, using the properties of class $\mathcal{KL}$ functions, \eqref{eq:hzt} ensures that $h(x(t),t)>0~\forall t\ge t_0$ and so $x(t) \in \mathcal C~\forall t\ge t_0$. Thus, $\mathcal C$ is forward invariant.
	\end{proof}
	\begin{remark}
		In case $\frac{\partial h}{\partial t} = 0$ and $L_g h(x,t) = 0$ we are not able to ensure the existence of a control input such that \eqref{eq:tvfcbf} holds, condition on which Lemma~\ref{thm:tvzcbf} relies. This case can be tackled by making use of a \textit{cascade} of control barrier functions and the result of Theorem~\ref{thm:nestedcbfs}.
	\end{remark}
	
	So far, we have described the condition in which the robots are executing the assigned task and we want them to keep their energy level $E_i$ within a certain interval $[\Emin,\Echg]$. This is realized, employing CBFs, by letting the robots reach regions of the state space where the field $I$ introduced in \eqref{eq:solarintensity} is larger than $I_\text{c}$. In these regions, as observed in Remark~\ref{rmk:solarintensity}, $\dot E_i > 0$ and the robots are charging.
	
	\changed{The use of CBFs will allow the robots to keep the energy stored in their batteries within the desired interval by synthesizing a controller solution of an optimization problem at each time instant. Nevertheless, although computationally efficient, this point-wise in time approach prevents the robots from planning their charging strategy. Therefore, in order to make sure that the robots leave the \textit{charging stations}---intended, in a broader sense, as the regions of the state space where $I(x_i,t)\ge I_\text{c}$---only when their battery is fully charged, i.\,e., when $E_i \approx \Echg$, we will make use of control Lyapunov functions (CLFs).} Similarly to what happens with CBFs encoding energy constraints, since the input $u_i$ does not directly show up in the expression of $\dot E_i$, a CLF defined to fully recharge the robots' battery will have relative degree higher than 1. Therefore, in the following, we will proceed analogously to what has been done for high relative degree CBFs to handle the case of high relative degree CLFs.
	
	Suppose that, besides the forward invariance of a set, one would like to stabilize the dynamical system \eqref{eq:controlaffine} around the origin $x=0$ using a CLF. Similarly to what has been done for the CBFs, the existence of a CLF $V : \R{n} \to \R{}$ suggests the definition of the following set:
	\begin{equation}
		K_V(x) = \left\{ u \in  U \;\left\vert\; - L_f V(x) - L_g V(x)\,u \right.\ge0 \right\}.
	\end{equation}
	It is easy to see how the choice of a control input $u \in K_V(x)$ will stabilize the system around $x = 0$.
	\begin{remark}
		In case $L_g V(x) = 0$, i.\,e., the relative degree of the Lyapunov function is greater than 1, we cannot ensure that the origin is stable.
	\end{remark}
	In Example~\ref{ex:cascadecbf} and in Theorem~\ref{thm:nestedcbfs}, a technique for dealing with high relative degree control barrier functions has been introduced: we will proceed here in a similar fashion. We first give an example that shows how to construct high-relative degree CLFs by employing CBFs. Then, we generalize this construction in Theorem~\ref{thm:nestedclfs}.
	\begin{example}[Constructing high relative degree CLFs using CBFs]
		Let us consider the nonlinear dynamical system in control affine form \eqref{eq:controlaffine} and a sufficiently smooth function $V : \R{n} \to \R{}$ with relative degree 2, i.\,e., $L_gV(x) = 0$ and $L_gL_fV(x) \neq 0$. In order for $V$ to be a CLF we must have $- L_f V(x) > 0$.
		We can then define the CBF $h_1(x) = - L_f V(x)$, and its superlevel set $\mathcal C_1 = \{ x \in \R{n} \vert - L_f V(x)\ge0 \}$, and let $u \in K_2^\prime(x) =  \left\{ u \in  U ~\left\vert~ - L_f^2 V(x) - L_gL_f V(x)\,u +\alpha(-L_fV(x)) \right.\ge0 \right\}$. where $\alpha$ is a locally Lipschitz class $\mathcal K$ function. This way, the set $\mathcal C_1$ can be rendered forward invariant. Consequently, the existence of the CLF $V(x)$ guarantees that the origin $x=0$ is (asymptotically) stable.
	\end{example}
	\begin{theorem}
		\label{thm:nestedclfs}
		Consider the dynamical system \eqref{eq:controlaffine}, a CLF $V(x)$ with relative degree $\rho$ defined with the objective of stabilizing the system state $x$ to $x^\ast$, and a CBF $h_\rho(x)$ \changed{which can be evaluated from $V(x)$ using the following recursive formula}:
		\begin{equation}
			\begin{cases}
				h_1(x) = -L_fV(x)\\
				h_{n+1}(x) = \dot h_n(x) + \alpha_n(h_n(x)),\quad\changed{1\le n<\rho},
			\end{cases}
			\label{eq:recursiveclf}
		\end{equation}
		where $\alpha_n$ are continuously differentiable extended class $\mathcal K$ functions. In addition, assume that $\{x^\ast\}$ is the largest invariant set in $\partial \mathcal C_1 = \{ x\,\vert\,h_1(x)=0\}$, boundary of the set $\mathcal C_1 = \{ x\,\vert\,h_1(x)\ge0\}$. Then, any Lipschitz continuous controller
		\begin{equation}
			\begin{aligned}
				u &\in K_\rho^\prime(x) = \biggl\{ u \in U ~\biggl\vert~ -L_f^\rho V(x) - L_gL_f^{\rho-1}V(x) u\\
				&+ \sum_{i=1}^{\rho-2}\sum_{\mathscr C \in \binom{\rho-2}{i}}\prod_{j\in \mathscr C}\frac{\partial \alpha_j}{\partial h_j}(-L_f^{\rho-1-i}V(x)) + \alpha_\rho(h_\rho(x)) \ge 0 \biggr.\biggr\},
			\end{aligned}
		\end{equation}
		where $\alpha_\rho$ is a locally Lipschitz class $\mathcal K$ function, will asymptotically stabilize the system to $x=x^\ast$.
	\end{theorem}
	\begin{proof}
		Given the properties of class $\mathcal K$ functions and using the fact that $h_1(x)=-L_fV(x)$ has relative degree $\rho-1$, the following expression of $\dot h_\rho$ can be derived:
		\begin{equation}
			\begin{aligned}
				\dot h_\rho(x) &= -L_f^\rho V(x) - L_gL_f^{\rho-1}V(x) u\\
				&+ \sum_{i=1}^{\rho-2}\sum_{\mathscr C \in \binom{\rho-2}{i}}\prod_{j\in \mathscr C}\frac{\partial \alpha_j}{\partial h_j}(-L_f^{\rho-1-i}V(x)).
			\end{aligned}
		\end{equation}
		The choice of $u \in K_\rho^\prime(x)$ will render the set $\mathcal C_\rho = \left\{ x \in \R{n} ~\vert~ h_\rho(x)\ge0 \right\}$ forward invariant. Similarly to what has been done in Theorem~\ref{thm:nestedcbfs}, by the recursive application of \eqref{eq:recursiveclf}, $\mathcal C_1$ is proven to be forward invariant. Then, by LaSalle's Theorem (Theorem 4.4 in \cite{khalil2002nonlinear}), as $\{x^\ast\}$ is the largest invariant set in $\partial \mathcal C_1$, one has that $x(t)\to x^\ast$ as $t\to\infty$.
	\end{proof}
	
	\section{Application to Robotic Tasks}
	\label{sec:applications}
	
	In view of what has been introduced in the previous section, in this section we show that the persistification of robotic tasks, specified in Definition~\ref{goal:gpt}, can be framed as a constrained optimization problem.
	
	In order for the robots to be able to perpetually execute the given task, their energy $E_i,~i=1,\ldots,N$ must be strictly greater than zero at each point in time. Moreover, in order to extend the battery life of the specific types of batteries used in robotic applications, the lower bound for the residual energy should not be too low in order to protect the battery from deep discharge \cite{garche2000battery}. Considering the robot model \eqref{eq:robotmodel}, the constraints to control the residual energy can be formally encoded by the following CBF related to robot $i$:
	\begin{equation}
		h_{i1}(z_i) = (\Echg-E_i)(E_i-\Emin),
		\label{eq:energycbf}
	\end{equation}
	where $\Echg$ and $\Emin$ are the upper and lower bounds between which we want the energy $E_i$ to be confined, corresponding to charged and depleted battery, respectively.
	
	Following the procedure adopted in Example~\ref{ex:cascadecbf}, we start by evaluating the time derivative of $h_{i1}(z_i)$:
	\begin{equation}
		\begin{aligned}
			\dot h_{i1}(z_i) &= \frac{\partial h_{i1}}{\partial t} + L_f h_{i1}(z_i) + L_g h_{i1}(z_i)\, u_i\\
			&= \left[ \frac{\partial h_{i1}}{\partial x_i}~\frac{\partial h_{i1}}{\partial E_i} \right] f(z_i,t) + \left[ \frac{\partial h_{i1}}{\partial x_i}~\frac{\partial h_{i1}}{\partial E_i} \right] g(z_i)\, u_i\\
			&= (\Echg+\Emin-2E_i)F(x_i,E_i,t),
		\end{aligned}
	\end{equation}
	which does not depend on $u_i$ as the relative degree of $h_{i1}(z_i)$ is 2. Therefore, we define the CBF $h_{i2}(z_i)$ as done in \eqref{eq:additionalcbf}, namely:
	\begin{equation}
		h_{i2}(z_i) = \dot h_{i1}(z_i) + \gamma_{i1} h_{i1}(z_i),
		\label{eq:energycbf2}
	\end{equation}
	in which the locally Lipschitz class $\mathcal K$ function has been chosen to be a linear function $\alpha_{i1}(s)=\gamma_{i1}s$, with $\gamma_{i1} > 0$.
	
	Using Theorem~\ref{thm:nestedcbfs} and Lemma~\ref{thm:tvzcbf}, we can define the set of control inputs $u_i$ that will render the set $\mathcal C_{i1} = \{ z_i\in\mathbb R^{n+1}~\vert~h_{i1}(z_i)\ge0 \} = \{ E_i\in\mathbb R~\vert~\Emin \le E_i\le \Echg \}$ forward invariant:
	\begin{equation}
		\begin{aligned}
			&K_{2i}(z_i) = \biggl\{ u_i \in U ~\biggl\vert~ \frac{\partial h_{i2}}{\partial t} + L_f^2 h_{i1}(z_i)\\
			& + L_g L_f h_{i1}(z_i)\,u_i + \gamma_{i1}L_f h_{i1}(z_i) + \gamma_{i2}h_{i2}(z_i) \ge 0 \biggr.\biggr\},
		\end{aligned}
		\label{eq:inputenergycbf}
	\end{equation}
	in which $\alpha_{i2}(h_{i2}(z_i)) = \gamma_{i2}h_{i2}(z_i)$, with $\gamma_{i2} > 0$, is the locally Lipschitz extended class $\mathcal K$ function in \eqref{eq:cascade2}. The expressions of $\frac{\partial h_{i2}}{\partial t}$, $L_f^2 h_{i1}(z_i)$ and $L_g L_f h_{i1}(z_i)$ are reported in full in Appendix~\ref{app:formulas}.
	
	Note that, due to the control affine form of \eqref{eq:robotmodelaffine}, $u_i\in K_{2i}(z_i)$ is an affine constraint in $u_i$ and therefore it can be written as:
	\begin{equation}
		A_{\text{CBF}i}(z_i) u_i \le b_{\text{CBF}i}(z_i),
		\label{eq:cbfconstraint}
	\end{equation}
	with $A_{\text{CBF}i}(z_i) = -L_g L_f h_{i1}(z_i)$ and $b_{\text{CBF}i}(z_i)=\frac{\partial h_{i2}}{\partial t} + L_f^2 h_{i1}(z_i) + \gamma_{i1}L_f h_{i1}(z_i) + \gamma_{i2}h_{i2}(z_i)$.
	\begin{remark}
		Recalling the expression of $\dot E$ introduced in \eqref{eq:energymodel}, the behavior resulting by enforcing the constraint \eqref{eq:cbfconstraint} is the following: when the battery level $E_i$ is getting close to its minimum value $\Emin$, robot $i$ will drive towards areas of the environment $\mathcal E$ where the value of the function $I(x_i,t)$ is such that $\dot E_i \ge 0$, i.\,e., robot $i$ starts recharging its battery.
	\end{remark}
	
	During operation, the battery of each robot continuously discharges according to the dynamics in \eqref{eq:energymodel}. The mere application of the constraint \eqref{eq:cbfconstraint} prevents the battery level to go lower than $\Emin$ or higher than $\Echg$. However, it does not ensure that the battery will be completely charged before robot $i$ leaves areas of the environment where $\dot E_i > 0$. \changed{This behavior is desirable for two reasons: (i) this kind of charging/discharging cycles will extend the life of robot batteries \cite{daniel2012handbook,garche2000battery}, and (ii) it is more efficient from the point of view of the time spent  by the robots recharging their batteries allowing them to deviate from the nominal task input for less time. A quantitative justification of the second reason is given in Section~\ref{subsec:exploration}}.
	
	This behavior can be encoded using a CLF whose objective is that of driving the energy $E_i$ to $\Echg$. We can then define the following CLF:
	\begin{equation}
		V_i(z_i) = (\Echg-E_i)^2,
		\label{eq:chargingclf}
	\end{equation}
	related to robot $i$, whose time derivative is given by
	\begin{equation}
		\begin{aligned}
			\dot V_i(z_i) &=\frac{\partial V_i}{\partial t} + L_f V_i(z_i) + L_g V_i(z_i)\, u_i\\
			&= \left[ \frac{\partial V_i}{\partial x_i}~\frac{\partial V_i}{\partial E_i} \right] f(z_i,t) + \left[ \frac{\partial V_i}{\partial x_i}~\frac{\partial V_i}{\partial E_i} \right] g(z_i)\, u_i\\
			&= -2(\Echg-E_i)F(x_i,E_i,t).
		\end{aligned}
	\end{equation}
	As in the case of $h_{i1}(z_i)$ in the previous section, here $V_i(z_i)$ has relative degree 2 and therefore its time derivative is not a function of the control input $u_i$. Proceeding as before, we define the CBF $h_{i1}(z_i)= -L_f V_i(z_i)$. Its superlevel set $\mathcal C = \{ z_i\in\mathbb R^{n+1}\,\vert\,h_{i1}(z_i) \ge 0 \} = \{ z_i\in\mathbb R^{n+1}\,\vert\,\dot V_i(z_i) \le 0 \}$ is the set in which the value of the function $V_i(z_i)$ is not increasing. Its boundary $\partial \mathcal C = \{ z_i\in\mathbb R^{n+1}\,\vert\,h_{i1}(z_i)=0\}$ is the set where $\dot V_i(z_i) = 0$ which, if $\dot E\neq 0$, coincides with the $n$-dimensional manifold $\{ z_i\in\mathbb R^{n+1}\,\vert\,E_i=\Echg\}$. Therefore, by Theorem~\ref{thm:nestedclfs}, if
	\begin{equation}
		\begin{aligned}
			u_i \in K_{2i}^\prime(z_i) =&\biggl\{ u_i \in U ~\biggl\vert~ \frac{\partial h_{i2}}{\partial t} - L_f^2 V_i(z_i) \\
			&- L_g L_f V_i(z_i)\,u_i +\gamma_{i2} h_{i2} \ge 0\biggr.\biggr\},
		\end{aligned}
		\label{eq:inputchargingcbf}
	\end{equation}
	with $\gamma_{i2}>0$, the value of $E_i$ will asymptotically converge to $\Echg$. The expressions of $\frac{\partial h_{i2}}{\partial t}$, $L_f^2 V_i(z_i)$ and $L_g L_f V_i(z_i)$ for the energy model \eqref{eq:energymodel}-\eqref{eq:wxet} are reported in full in Appendix~\ref{app:formulas}. Note that $u_i\in K_{2i}^\prime(z_i)$ is an affine constraint in $u_i$, and therefore it can be written as:
	\begin{equation}
		A_{\text{CLF}i}(z_i) u_i \le b_{\text{CLF}i}(z_i),
		\label{eq:clfconstraint}
	\end{equation}
	with $A_{\text{CLF}i}(z_i) = L_g L_f V_i(z_i)$ and $b_{\text{CLF}i}(z_i)=\frac{\partial h_{i2}}{\partial t} - L_f^2 V_i(z_i)+\gamma_{i2} h_{i2}(z_i)$.
	
	In order to combine the CBF constraints \eqref{eq:cbfconstraint} and the CLF constraints \eqref{eq:clfconstraint}, the following nonlinear program can be formulated:
	\begin{equation}
		\begin{aligned}
			u^\ast = \argmin_{u,\delta}  &\norm{u-\hat u}^2+\delta\tr\kappa\delta\\
			\st & \begin{aligned}
				\begin{bmatrix}
					A_\text{CBF}(z_i) & 0_N\\
					A_\text{CLF}(z_i) & -I_N
				\end{bmatrix} &\begin{bmatrix}u\\ \delta\end{bmatrix} \le
				\begin{bmatrix}
					b_\text{CBF}(z_i)\\
					b_\text{CLF}(z_i)
				\end{bmatrix}
			\end{aligned},
		\end{aligned}
		\label{eq:qp}
	\end{equation}
	where
	\begin{equation}
		\begin{gathered}
			A_\text{CBF}(z_i) = \mathrm{diag}\left( A_{\text{CBF}1}(z_i),~\ldots,~A_{\text{CBF}N}(z_i) \right),\\
			A_\text{CLF}(z_i) = \mathrm{diag}\left( A_{\text{CLF}1}(z_i),~\ldots,~A_{\text{CLF}N}(z_i) \right),\\
			b_\text{CBF}(z_i) = \begin{bmatrix}
				b_{\text{CBF}1}(z_i)\\
				\vdots\\
				b_{\text{CBF}N}(z_i)
			\end{bmatrix},~
			b_\text{CLF}(z_i) = \begin{bmatrix}
				b_{\text{CLF}1}(z_i)\\
				\vdots\\
				b_{\text{CLF}N}(z_i)
			\end{bmatrix},
		\end{gathered}
	\end{equation}
	$u = [u_1^T,\ldots,u_N^T]$, and $I_N$ and $0_N$ are $N\times N$ identity and zero matrices, respectively. Moreover, $\delta = \left[\delta_1,\ldots,\delta_N\right]\tr \in \mathbb R^N$ is a vector of relaxation parameters introduced to make the constraints in \eqref{eq:qp} always feasible. The matrix $\kappa = \mathrm{diag}(\kappa_i)$ is a diagonal, positive definite, weighting matrix for $\delta$. Furthermore, the nominal input $\hat u$ is what encodes the task introduced in Definition~\ref{goal:gpt}.
	
	Both \eqref{eq:cbfconstraint} and \eqref{eq:clfconstraint} are affine functions of the optimization variables. Moreover, the cost $\norm{u-\hat u}^2+\delta\tr\kappa\delta$ is a convex quadratic form. \changed{Hence, \eqref{eq:qp} is a convex quadratic program (QP) and, as such, can be efficiently solved---see, e.\,g., \cite{boyd2004convex}---and employed to generate controllers in an online fashion as shown, for instance, in \cite{ames2013towards}}.
	\begin{remark}
		As discussed in Section~\ref{sec:problem}, the proposed persistification approach also works when the $N$ robots are not homogeneous, i.\,e., when they are characterized by different dynamic models. In fact, in this case, once the constraints of the optimization program \eqref{eq:qp} have been changed accordingly, the solution $u^\ast$ guarantees the heterogeneous multi-robot system persistently executes the task characterized by the nominal input $\hat u$.
	\end{remark}
	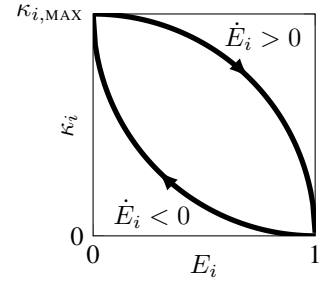
\begin{figure}[th]
		\centering
		\begin{tikzpicture}
		%	\large
		\begin{axis}
		[
		xlabel={$E_i$},
		ylabel={$\kappa_i$},
		no marks,
		xtick={0,1},
		ytick={0,1},
		xticklabels={0,1},
		yticklabels={0,$\kappa_{i,\text{\sc max}}$},
		enlargelimits=false,
		ylabel shift = -30 pt,
		xlabel shift = -10 pt,
		width=0.25\textwidth, % image width
		height=0.25\textwidth % image height
		]
		\addplot[
		line width=2,
		domain=0:1,
		samples=100,
		black,
		postaction={decorate, decoration={markings,
				mark=at position 0.5 with {\arrow[ultra thick]{latex};},
		}}
		] {sqrt(1-x^2)};
		\addplot[
		line width=2,
		domain=0:1,
		samples=100,
		black,
		postaction={decorate, decoration={markings,
				mark=at position 0.5 with {\arrowreversed[ultra thick]{latex};},
		}}
		] {1-sqrt(1-(x-1)^2)};
		\node[anchor=west] () at (axis cs:0.55,0.9) {$\dot E_i > 0$};
		\node[anchor=west] () at (axis cs:0.05,0.1) {$\dot E_i < 0$};
		\end{axis}
		\end{tikzpicture}
		\caption{Example of the function $\kappa_i(E_i)$ that can be employed in order to let robot $i$ charge its battery up to $\Echg$ once it has started charging. The two branches of the curve are labeled with $\dot E_i>0$ and $\dot E_i<0$: at a given value of $E_i$, the value of $\kappa_i$ is higher for a robot that is recharging its battery ($\dot E_i>0$) compared to the one of a robot for which $\dot E_i<0$. This way, once a robot starts recharging its battery, the weight of its corresponding component of the vector $\delta$ in \eqref{eq:qp} is larger. This ensures that the robot charges its battery until $\Echg$.}
		\label{fig:kappa}
	\end{figure}
	\begin{remark}
		In order to achieve the desired charging behavior, discussed in Section~\ref{sec:control}, $\kappa_i$ can be made a function of $E_i$ in such a way that robot $i$ charges up to $E_i = \Echg$ once it started charging and, at the same time, discharge down to $E_i = \Emin$ while operating. A candidate mapping $\kappa_i(E_i) : [0,1] \to [0,\kappa_{i,\text{\sc max}}]$ is depicted in Fig.~\ref{fig:kappa}. Note that the function $\kappa_i$ changes according to the sign of $\dot E_i$. This way, the weight of the corresponding relaxation parameter $\delta_i$ changes when the energy $E_i$ reaches $\Echg$ or $\Emin$.
	\end{remark}
	
	We conclude this section stating the proposition which ensures the task persistification as defined in Definition~\ref{eq:persistification}.
	\begin{proposition}
		The control law $u^\ast$, solution of the QP \eqref{eq:qp}, makes the robots execute the persistified task corresponding to the task encoded through the control input $\hat u$.
	\end{proposition}
	\begin{proof}
		Follows from the application of Theorem~\ref{thm:nestedcbfs}, Lemma~\ref{thm:tvzcbf} and Theorem~\ref{thm:nestedclfs}.
	\end{proof}
	
	In the next section, the developed theoretical control framework will be validated by means of simulations and experiments. Two robotic tasks are introduced and the results of the their persistent implementation are reported.
	
	\section{Simulations and Experiments}
	\label{sec:experiments}
	
	In this section, two robotic tasks, whose persistent application is particularly relevant, are presented to showcase the persistification strategy presented in this paper. The tasks consist in environment exploration and environment surveillance. Both tasks are typically required to be executed for a long period of time. In the case of environment exploration, the long execution time can be due to the size and/or the dynamic nature of the environment to explore \cite{girdhar2016modeling,bellingham2007robotics}. As regards the environment surveillance, the time-scale of the observed environment phenomena is the factor that determines the length of the task. Nowadays, longevity is still a limiting factor for the deployment of robotic systems for environment surveillance and monitoring, as described in \cite{dunbabin2012robots}. The persistification of these two tasks through the control framework described in the previous section is discussed in Section~\ref{subsec:exploration} and Section~\ref{subsec:surveillance}, respectively.
	
	\subsection{Environment Exploration}
	\label{subsec:exploration}
	
	The first application that is considered is that of environment exploration. There are many approaches to this task in literature, as discussed in Section~\ref{sec:intro}, and many solutions have been proposed. Here we consider the one presented in \cite{miller2013trajectory}, since the result of a trajectory optimization problem provides directly the nominal inputs to the robots. The optimal trajectories are evaluated by minimizing the distance from ergodicity \cite{mathew2011metrics}. This results in a trajectory that, instead of maximizing information in a greedy way, distributes information according to its probability density function defined over the environment.
	
	Following what is presented in \cite{miller2013trajectory}, let us start by defining the following ergodic metric as in \cite{mathew2011metrics}:
	\begin{equation}
		\epsilon = \sum_{k=0}^{K} \Lambda_k \abs{c_k-\varphi_k}^2,
		\label{eq:ergodicmetric}
	\end{equation}
	where $c_k$ are the time-averaged Fourier coefficients of the trajectory, $\varphi_k$ are the Fourier coefficients of a spatial distribution of information $\phi : \mathcal E \mapsto \R{}_+$, $\mathcal E$ being the environment to explore, and
	\begin{equation}
		\Lambda_k = \frac{1}{(1+\xi\tr \xi)^{\frac{3}{2}}},
	\end{equation}
	with $\xi \in \mathscr Z = \{ 0, 1, \ldots, K-1\} \times \{ 0, 1, \ldots, K-1 \}$, $K$ being the number of employed Fourier basis functions.
	
	By minimizing this ergodic measure at time $\bar t$ over a time horizon $T$, the nominal trajectory $\hat x_i(t)$ for $t \in [\bar t,~\bar t+T]$ of each robot is obtained. Solving this optimization problem at every time instant, in a model predictive control fashion, provides the nominal input $\hat u_i$ that is to be executed at time $\bar t$ in order to track the trajectory $\hat x_i$ \cite{miller2013trajectory}. This $\hat u_i$ can be plugged in the QP defined in \eqref{eq:qp} allowing, this way, a straightforward application of the persistification framework presented in this paper to the environment exploration task.
	
	The persistent environment exploration task has been implemented and tested in a simulation environment. For the simulated experiment a planar robot is given the task of exploring an environment $\mathcal E$ on which a spatial distribution of information has been defined. The information is assumed to be distributed according to the Gaussian density function
	\begin{equation}
		\phi : x \in \mathcal E \subset \R{2} \mapsto e^{-\frac{\norm{x-x_\text{o}}^2}{\sigma^2}} \in \R{}_+,
	\end{equation}
	where, for the experiments, the following values are used: $x_\text{o}=[0,0]^T$ and $\sigma^2=0.1$. The time-varying environment field $I$ is modeled as a mixture of time-varying Gaussians of the following form:
	\begin{equation}
		I(x,t) = e^{-\norm{x-M_1(t) x_\text{c}}^2} + e^{-\norm{x-M_2(t) x_\text{c}}^2},
	\end{equation}
	where $x_\text{c} = [1,1]^T$ and
	\begin{equation}
		M_1(t) = \begin{bmatrix}
			-1 & 0\\
			0 & \sin{(2t)}
		\end{bmatrix},\quad
		M_2(t) = \begin{bmatrix}
			\sin{(2t)} & 0\\
			0 & 1
		\end{bmatrix}.
	\end{equation}
	In the case of robots that are able to exploit solar power to recharge their batteries, this choice of the function $I$ simulates the sunlight intensity that is characterized by a periodic expression over a spatially fixed environment. The other values required to model $\dot E_i$ \eqref{eq:energymodel} in \eqref{eq:robotmodel} are set to $I_\text{c}=0.85$ and $\lambda = 3$.
	
	The optimization problem aimed at minimizing the ergodic cost \eqref{eq:ergodicmetric} is solved offline and the resulting trajectory is given to the robot as a reference for the tracking controller. The input required to track the trajectory is wrapped by the QP \eqref{eq:qp} in such a way that the robot explores the environment while satisfying at the same time the energy constraint that prevents its battery level to go below the lower threshold $\Emin$. The result of this controller is a persistified environment exploration.
	
	\begin{figure}
		\centering
		\subfloat[]{\label{subfig:bigfig_simulations_1}\includegraphics[width=0.245\textwidth]{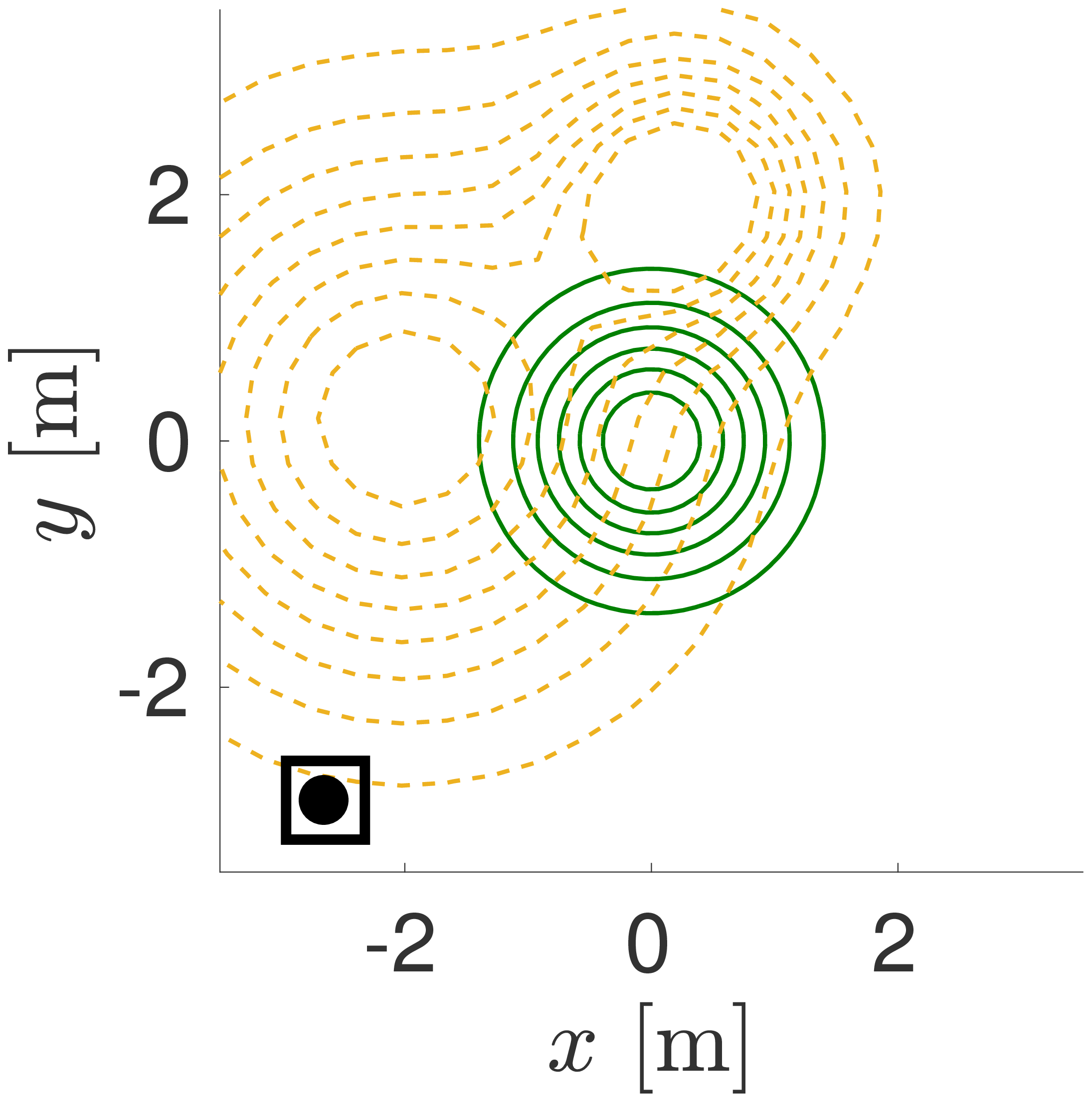}}%
		\subfloat[]{\includegraphics[width=0.245\textwidth]{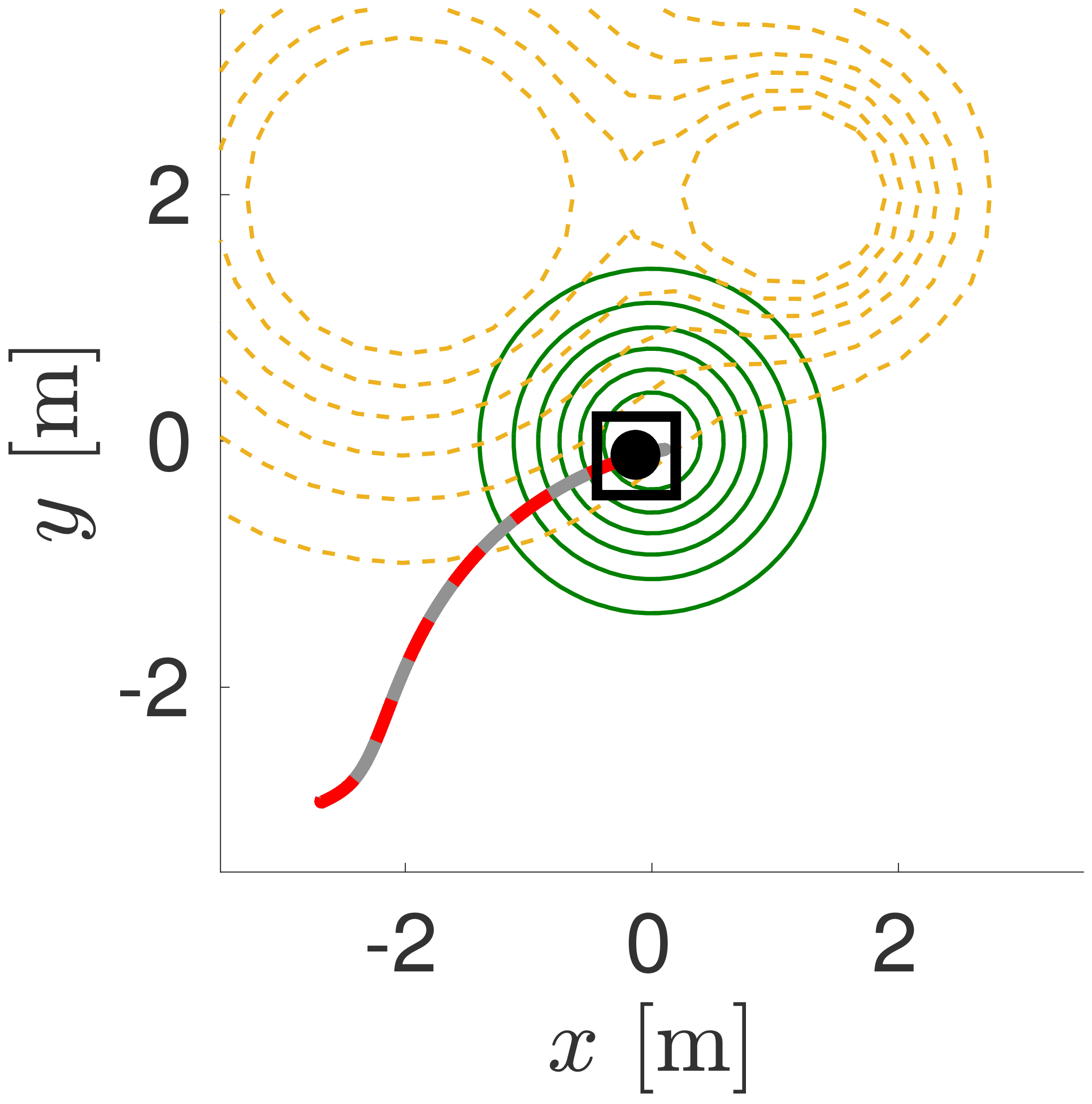}}\\
		\subfloat[]{\includegraphics[width=0.245\textwidth]{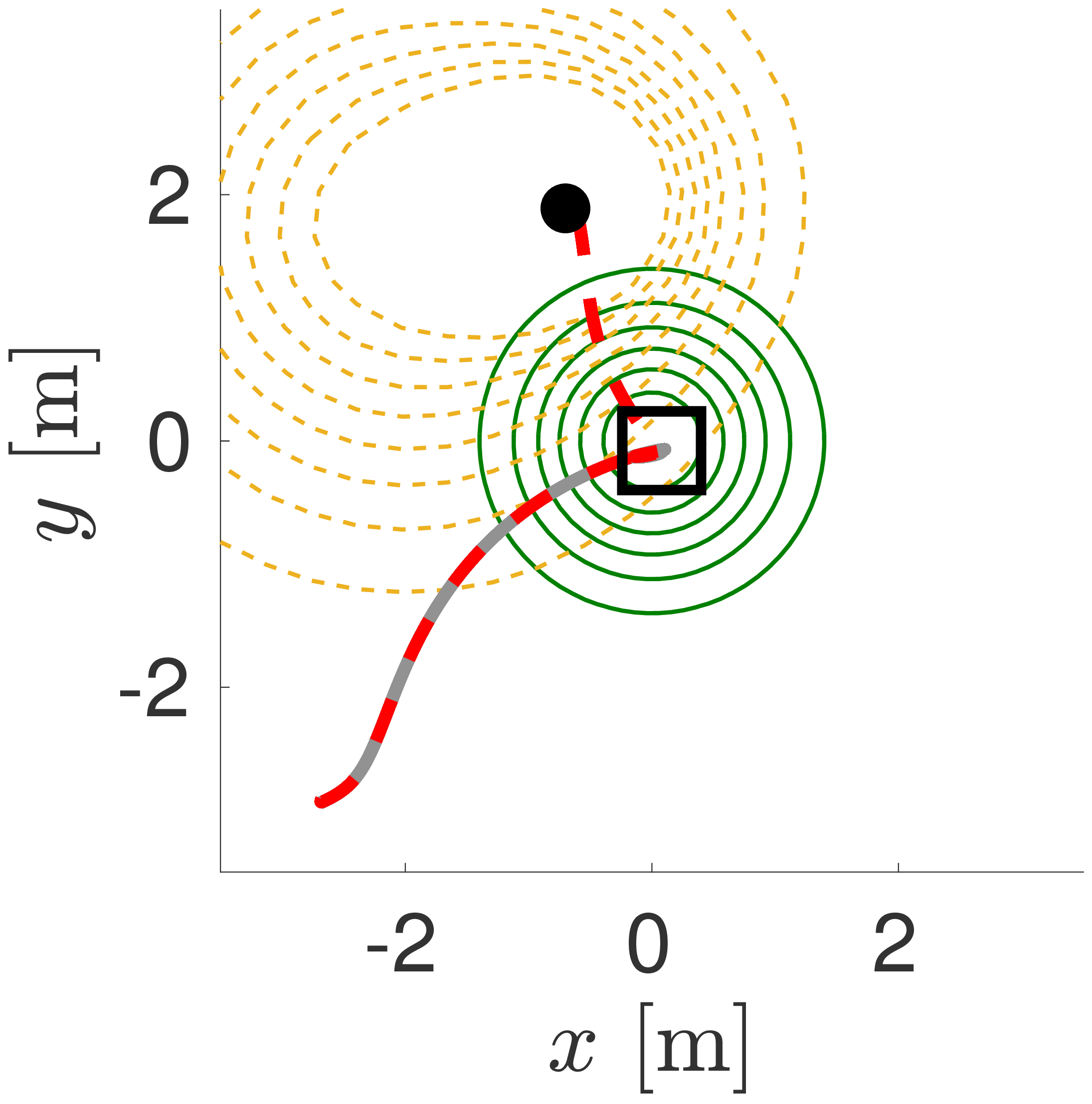}}%
		\subfloat[]{\includegraphics[width=0.245\textwidth]{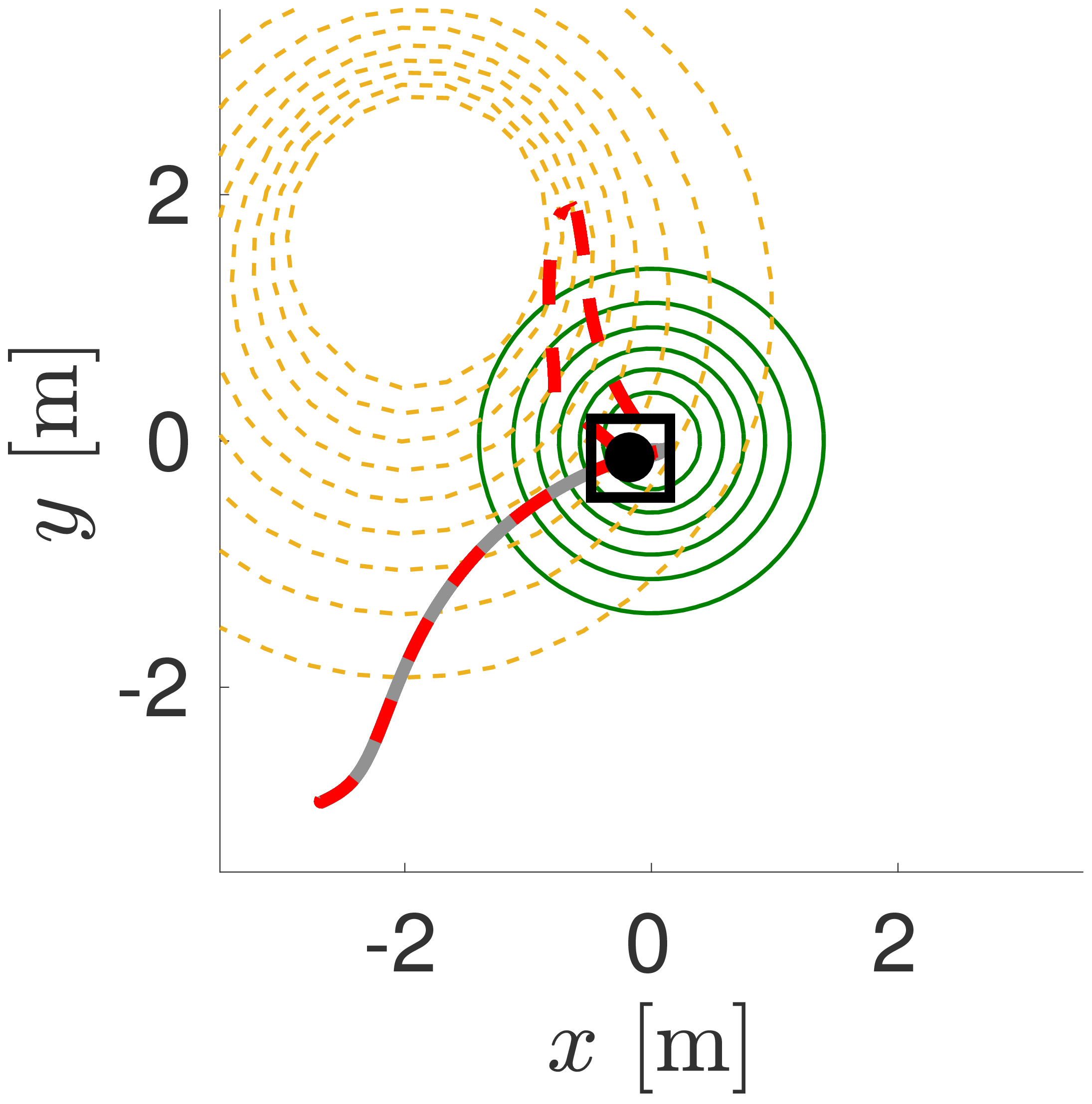}}\\
		\subfloat[]{\label{subfig:bigfig_simulations_5}\includegraphics[width=0.245\textwidth]{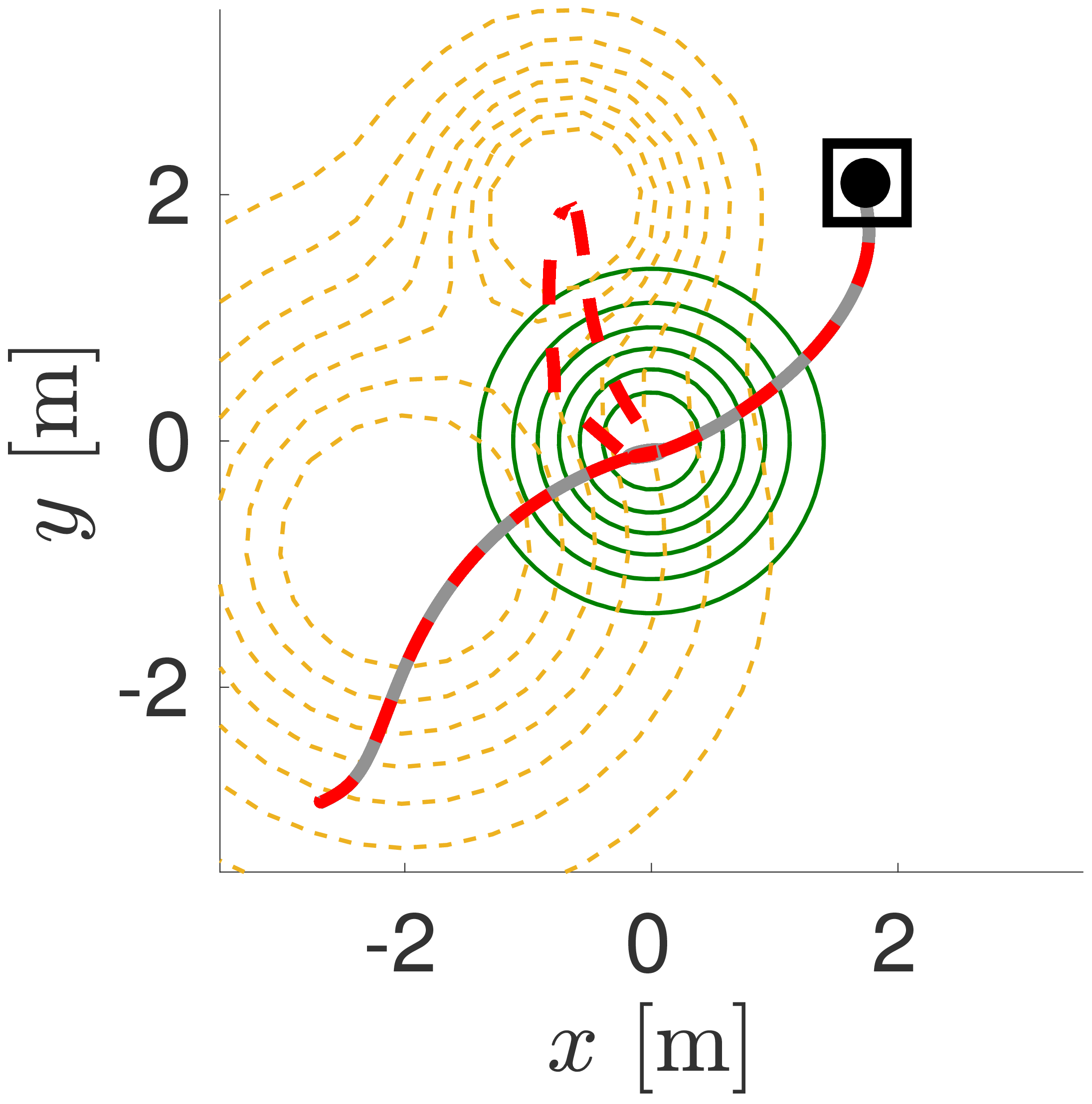}}
		\caption{Sequence of images recorded during the course of the environment exploration simulated experiment. The contour plots of the information distribution function $\phi$ and the environment field $I$ are depicted as green thin solid lines and yellow thin dashed lines, respectively. The position tracked by the robot under the nominal control input is represented as a black square, while the actual position of the robot is shown as a black dot. The nominal and actual trajectories are depicted as a gray thick solid line and a red thick dashed line, respectively. In order to persistently explore the environment, the robot follows the nominal input as long as its energy level is high enough. When its battery is depleting, it moves towards regions of the environment where the value of the time-varying field $I$ is such that its energy starts increasing.}
		\label{fig:bigfig_simulations}
	\end{figure}
	Figures~\ref{subfig:bigfig_simulations_1} to \ref{subfig:bigfig_simulations_5} show a sequence of snapshots taken during the course of the environment exploration experiment described above. The contour plot of the function $\phi$ is depicted as green thin solid lines, while the contour plot of the environment field $I$ is represented by the yellow thin dashed lines. The position that the robot is tracking under the nominal control input is depicted as a black square, whereas its actual position executing the controller \eqref{eq:qp} is represented by a black circle. Furthermore, nominal and executed trajectories are represented by a gray thick solid line and a red thick dashed line, respectively.
	
	\begin{figure*}[th]
		\centering
		\subfloat[]{\label{subfig:pdfphi}
			\includegraphics[width=0.325\textwidth]{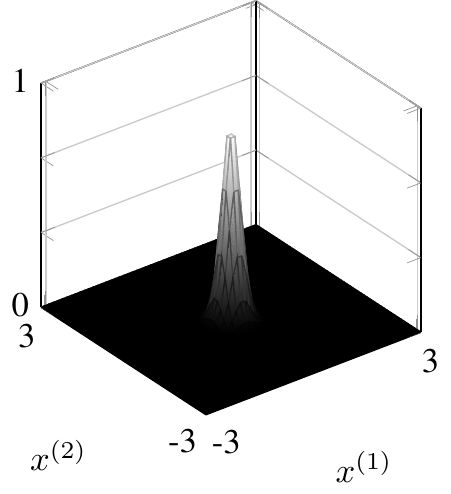}}\hfill %subfloat
		\subfloat[]{\label{subfig:pdfenergy}
			\includegraphics[width=0.325\textwidth]{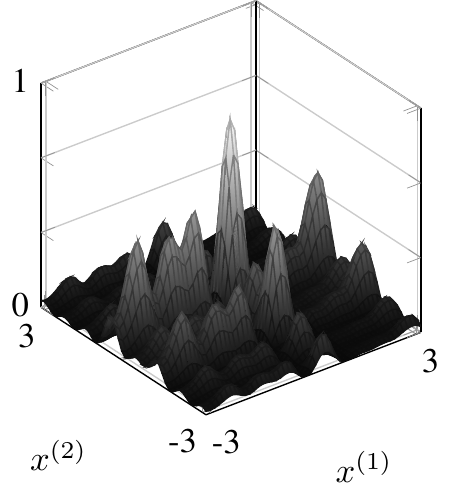}}\hfill %subfloat
		\subfloat[]{\label{subfig:pdfnoenergy}
			\includegraphics[width=0.325\textwidth]{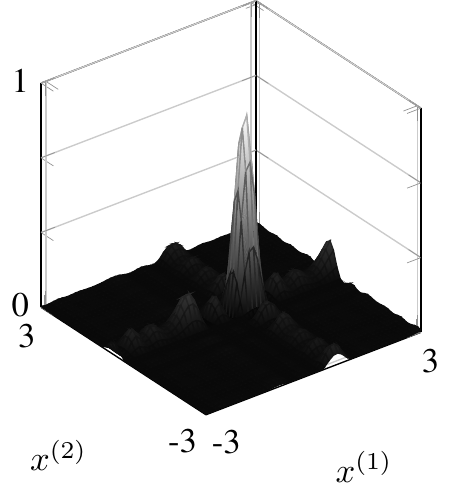}} % subfloat
		\caption{Comparison between the spatial probability density function $\phi$, in \protect\subref{subfig:pdfphi}, and the probability density function obtained averaging over time the ergodic trajectory resulting from the implementation of the persistified environment exploration, in \protect\subref{subfig:pdfenergy}; \protect\subref{subfig:pdfnoenergy} depicts the probability density function representing the time-averaged optimized ergodic trajectory obtained without taking into account energy constraints. $x^{(1)}$ and $x^{(2)}$ are the two components of the state vector $x \in \mathcal E \subset \mathbb R^2$.}
		\label{fig:trajpdf}
	\end{figure*}
	Figure~\ref{fig:trajpdf} compares the probability density function representing the spatial information distribution $\phi$ (Fig.~\ref{subfig:pdfphi}) with the probability density functions for the time-averaged optimized trajectory with energy constraints (Fig.~\ref{subfig:pdfenergy}) and without energy constraints (Fig.~\ref{subfig:pdfnoenergy}). Even though the latter more closely matches the spatial distribution $\phi$, it does not take into account that the robot has a finite availability of energy.
	
	\begin{figure}[th]
		\centering
		\includegraphics[width=.49\textwidth]{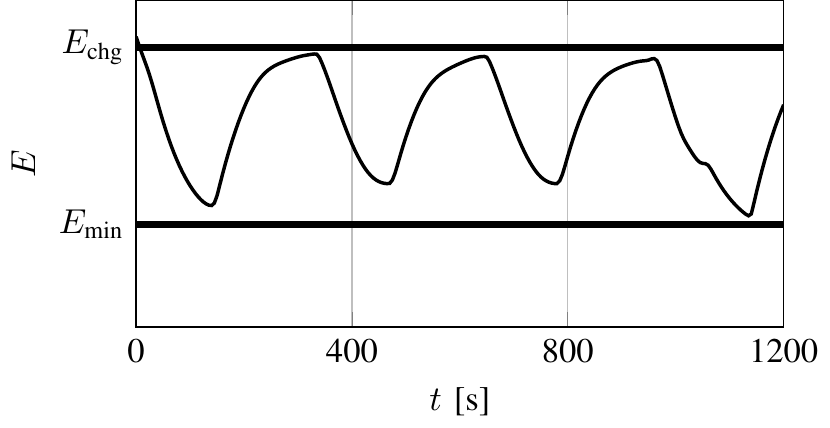}
		\caption{Simulated battery level of the robot during the course of the persistified exploration experiment. Employing the persistification strategy presented in this paper, achieved by executing the control input solution of the QP~\eqref{eq:qp}, the robot energy (thin line) is constrained within the bounds $\Emin$ and $\Echg$.}
		\label{fig:ergodicenergy}
	\end{figure}
	
	In Fig.~\ref{fig:ergodicenergy}, the energy level of the robot during a persistified long-term exploration experiment are reported. The application of the control framework presented in this paper is demonstrated to be successful in persistifying the robotic exploration. This is realized by keeping the robot energy level constrained above a minimum value in an optimal way by means of the QP \eqref{eq:qp}. This way, the robot is completely free of tracking the ergodic trajectory given as input to its motion controller, as long as its battery level is above the lower threshold $\Emin$ and below the upper threshold $\Echg$, depicted as thick solid lines in Fig.~\ref{fig:ergodicenergy}.
	
	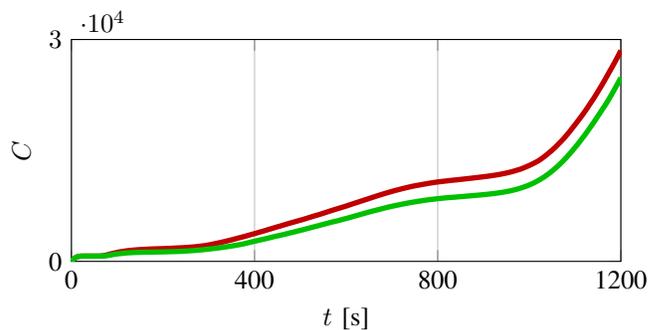
\begin{figure}[th]
		\centering
		\begin{tikzpicture} % scale and rotate
		%	\Large % axis font size
		\begin{axis}
		[
		no marks, % remove marks
		xlabel={$t$ [s]}, % xlabel,
		xtick={0,400,800,1200}, % xtick
		xticklabels={0,400,800,1200}, % xticklabel
		ylabel={$C$}, % ylabel
		ytick={0,30000},
		yticklabels={0,$3$},
		xmin=0, % xlim
		xmax=1200, % xlim
		ymin=0, % ylim
		ymax=30000, % ylim
		enlarge x limits=-1, % x axis tight
		enlarge y limits=-1, % y axis tight
		%		axis equal image, % axis equal
		grid=both, % grid on
		%		xmajorgrids, % x grid on
		%		ymajorgrids, % y grid on
		%		minor tick num=2, % grid minor
		%		legend entries={one-column plot,test12,test34,test56}, % legend
		%		legend style={nodes=right}, % legend style
		%		legend pos= north west, % legend position
		width=0.49\textwidth, % image width
		height=0.25\textwidth % image height
		]
		\addplot [line width=2pt, color=clr1] table [x=t, y=woV, col sep=space]{data/withWithoutCLF.txt};
		\addplot [line width=2pt, color=clr2] table [x=t, y=wV, col sep=space]{data/withWithoutCLF.txt};
		\end{axis}
		\end{tikzpicture}
		\caption{\changed{Comparison between persistent task execution with and without battery recharging constraints. The green curve depicts the value of $C$, defined in \eqref{eq:Ct}, in which the robot input $u(t) = u^*(t)$, solution of \eqref{eq:qp}. The red curve has been obtained by letting the robot execute the input $u(t)$ solution of \eqref{eq:qp} from which the constraint \eqref{eq:clfconstraint} corresponding to battery recharing has been removed. As a result, the value of $C$ in the latter case is higher than the one in the former, i.\,e. the robot spends more time deviating from the nominal input $\hat u(t)$ in order to visit charging stations in the environment and prevent its energy from depleting.}}
		\label{fig:wwoCLF}
	\end{figure}
	
	\changed{In Section~\ref{sec:applications}, CLF constraints were introduced with the objective of completely recharging the battery of the robots before leaving areas of the environment where $\dot E>0$, referred to as charging stations. This leads to a more efficient task execution in the sense that the robots overall spend less time recharging their batteries. This condition allows them to deviate from the nominal task input of a smaller amount over the course of the experiment. In Fig.~\ref{fig:wwoCLF}, the value of $C$ at time $t$ represents the integral over the time interval $[0,t]$ of the difference between the nominal control input $\hat u$---corresponding to the exploration task---and the input $u$ executed by the robot:
		\begin{equation}
			\label{eq:Ct}
			C(t)=\int_0^t \norm{u(\tau)-\hat u(\tau)}^2 d\tau.
		\end{equation}
		The green curve is generated letting the robot execute the input $u(t)=u^*(t)$, solution of \eqref{eq:qp}. The red curve corresponds to the situation where the robot executes the input solution of \eqref{eq:qp} in which the CLF constraint \eqref{eq:clfconstraint} for battery recharging has been removed. As can be seen, the battery recharging constraint allows the robot to visit charging stations less often, so that its input deviates from the nominal one of a smaller amount over the course of the task execution.}
	
	\subsection{Environment Surveillance}
	\label{subsec:surveillance}
	
	The second application that is considered as showcase is environment surveillance. The employment of mobile sensors improves coverage and data gathering performances compared to static sensors, whose positions are determined based on offline optimization algorithms. In this sense, mobility can allow a more efficient estimation of time-varying information fields. However, this comes at the price of higher power consumption. Most of the approaches developed so far assume that the mobile sensors are able to move for an unlimited amount of time. The control framework presented in this paper can be used to persistify such surveillance tasks.
	
	The task of environment surveillance can be framed as a sensor coverage control problem, that is an instance of the broader optimal sensor placement problem whose applications can be found in many other disciplines, such as \cite{okabe1997locational}. As in Section~\ref{subsec:exploration}, let the map $\phi : \mathcal E \to \R{}_+$ represent a spatial distribution density function. This can be interpreted as a measure of the information spread over the environment $\mathcal E$ or the probability that an event can take place at a location $x \in \mathcal E$. Moreover, let us define the locational optimization function as in \cite{cortes2004coverage}:
	\begin{equation}
		\mathcal H(\mathcal X, \mathcal W) = \sum_{i=1}^{N} \int_{W_i} \sigma(\norm{x-x_i}) \phi(x) dx,
		\label{eq:loccost}
	\end{equation}
	where $\mathcal X = \{ x_1, \ldots, x_N \}$ are the positions of the $N$ robots present in the environment $\mathcal E$, $\mathcal W = \{ W_1, \ldots, W_N \}$ is a partition of $\mathcal E$, $\sigma :  \R{}_+ \to \R{}_+$ is a non-decreasing differentiable function describing the degradation in the sensing performances of the robots. Proceeding as in \cite{cortes2004coverage}, we aim at minimizing $\mathcal H(\mathcal X, \mathcal W)$ with respect to both $\mathcal X$ and $\mathcal W$. The minimization with respect to the environment partition $\mathcal W$ leads to $\mathcal W = \mathcal V = \{ V_1, \ldots, V_N \}$ \cite{okabe1997locational}, $\mathcal V$ being the Voronoi partition of $\mathcal E$ defined by the Voronoi cells
	\begin{equation}
		V_i = \{ x \in \mathcal E ~\left\vert~ \norm{x-x_i} \le \norm{x-x_j}~\forall i \neq j \right. \}.
	\end{equation}
	As far as the minimization with respect to the robot locations $\mathcal X$ is concerned, considering the single integrator dynamics of the robots, and setting $\sigma(\norm{x-x_i})= \norm{x-x_i}^2$ allows us to define the following gradient descent update step to be used as a control law for moving the robots \cite{mesbahi2010graph}:
	\begin{equation}
		\hat u_i = k_p (C_{V_i}-x_i).
		\label{eq:coverage}
	\end{equation}
	In \eqref{eq:coverage}, $k_p > 0$ is a proportional gain and $C_{V_i}$ is the centroid of the $i$-th Voronoi cell defined as:
	\begin{equation}
		C_{V_i} = \dfrac{\int_{V_i} x \phi(x) dx}{\int_{V_i} \phi(x) dx}.
	\end{equation}
	In case the map $\phi$ is time-varying, the extension presented in \cite{lee2015multirobot} and \cite{santos2019dencetralized} can be employed. The coverage task with $\hat u = [\hat u_1\tr, \ldots, \hat u_N\tr]\tr$ and $\hat u_i$ given by \eqref{eq:coverage}, can be persistified by the implementation of the optimization-based controller solution of the QP in \eqref{eq:qp}.
	
	The persistified environment surveillance has been implemented and deployed on the Robotarium, a remotely accessible swarm robotics research testbed \cite{pickem2017robotarium}. A team of 7 differential-drive robots attempts to cover the $4.5\times3.5$\,m testbed area by making use of the coverage control introduced above. The Robotarium is endowed with wireless charging stations, allowing the modeling of the charging field $I$ given in \eqref{eq:solarintensity} by means of bump-like functions as depicted in Fig.~\ref{fig:lumpedenergysources}. The robot model is a single integrator model, where we can directly control the robots' velocities utilizing the proportional control law given in \eqref{eq:coverage}. This input is wrapped by the QP \eqref{eq:qp} in order to satisfy the energy constraints. This results in a persistified environment surveillance.
	
	\begin{figure}
		\centering
		\subfloat[]{\label{subfig:bigfig_experiments_1}\includegraphics[trim={12cm 0cm 0cm 0cm},clip,width=0.35\textwidth]{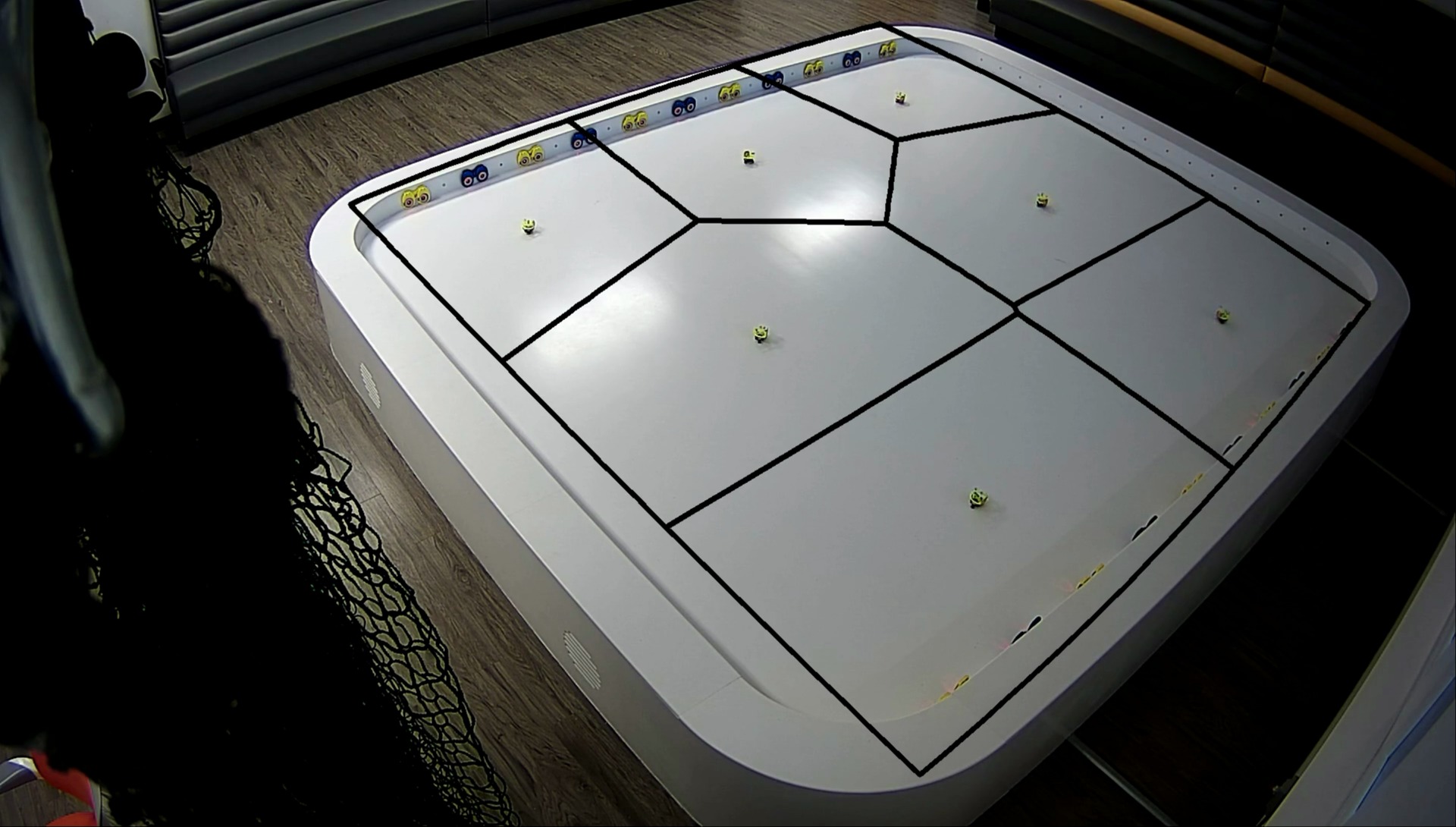}}\\
		\subfloat[]{\label{subfig:bigfig_experiments_2}\includegraphics[trim={12cm 0cm 0cm 0cm},clip,width=0.35\textwidth]{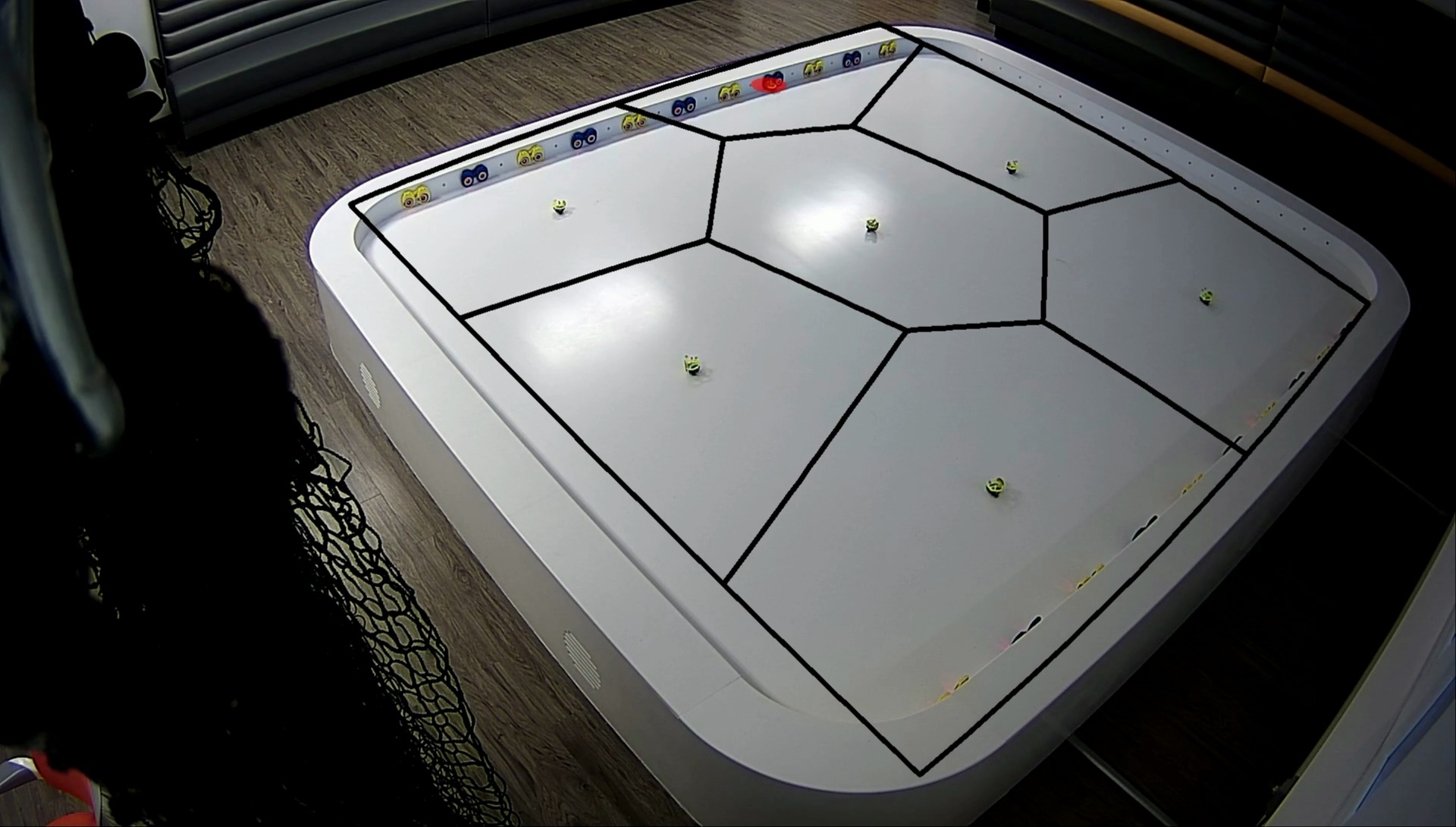}}\\
		\subfloat[]{\label{subfig:bigfig_experiments_3}\includegraphics[trim={12cm 0cm 0cm 0cm},clip,width=0.35\textwidth]{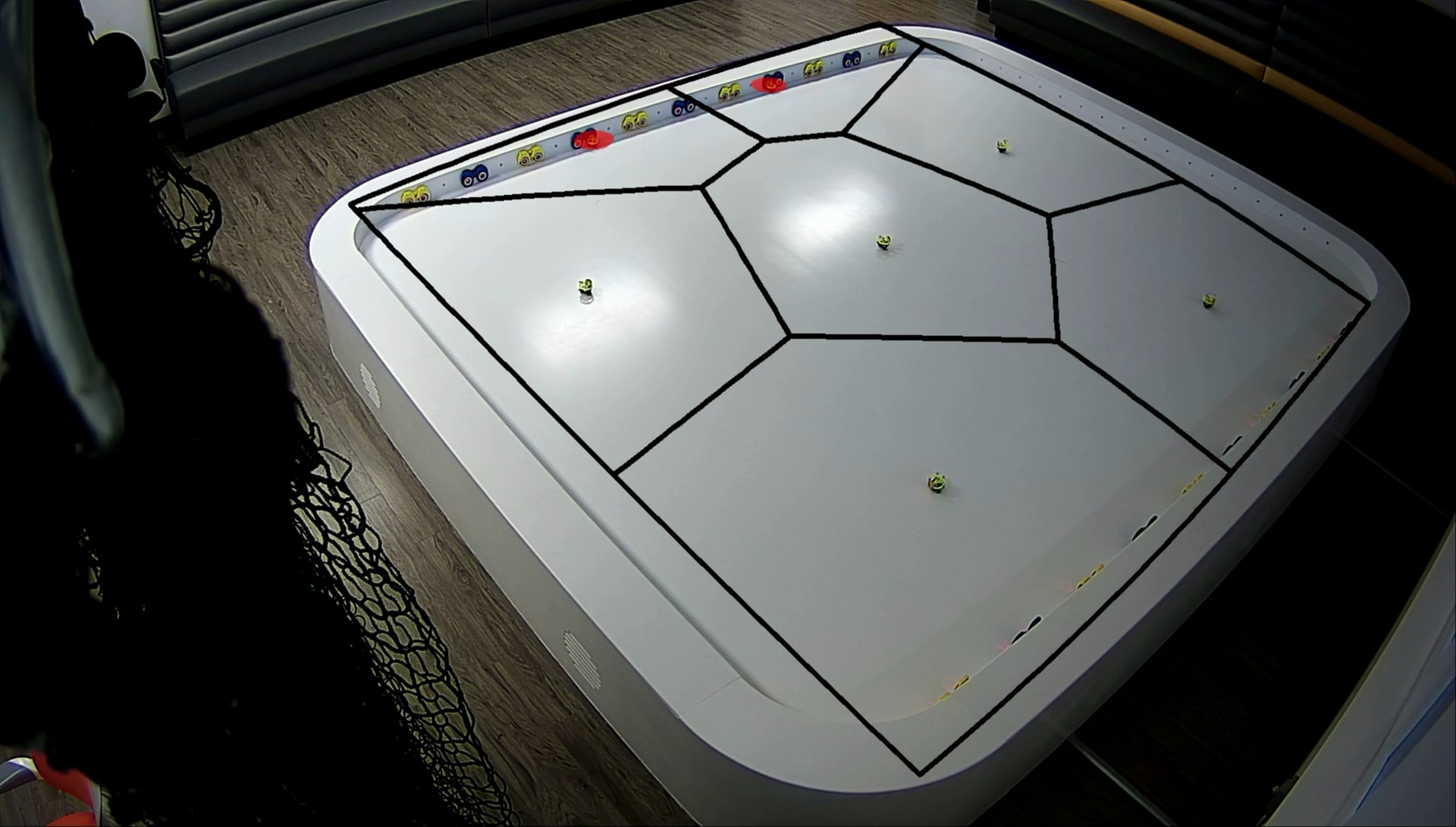}}\\
		\subfloat[]{\label{subfig:bigfig_experiments_4}\includegraphics[trim={12cm 0cm 0cm 0cm},clip,width=0.35\textwidth]{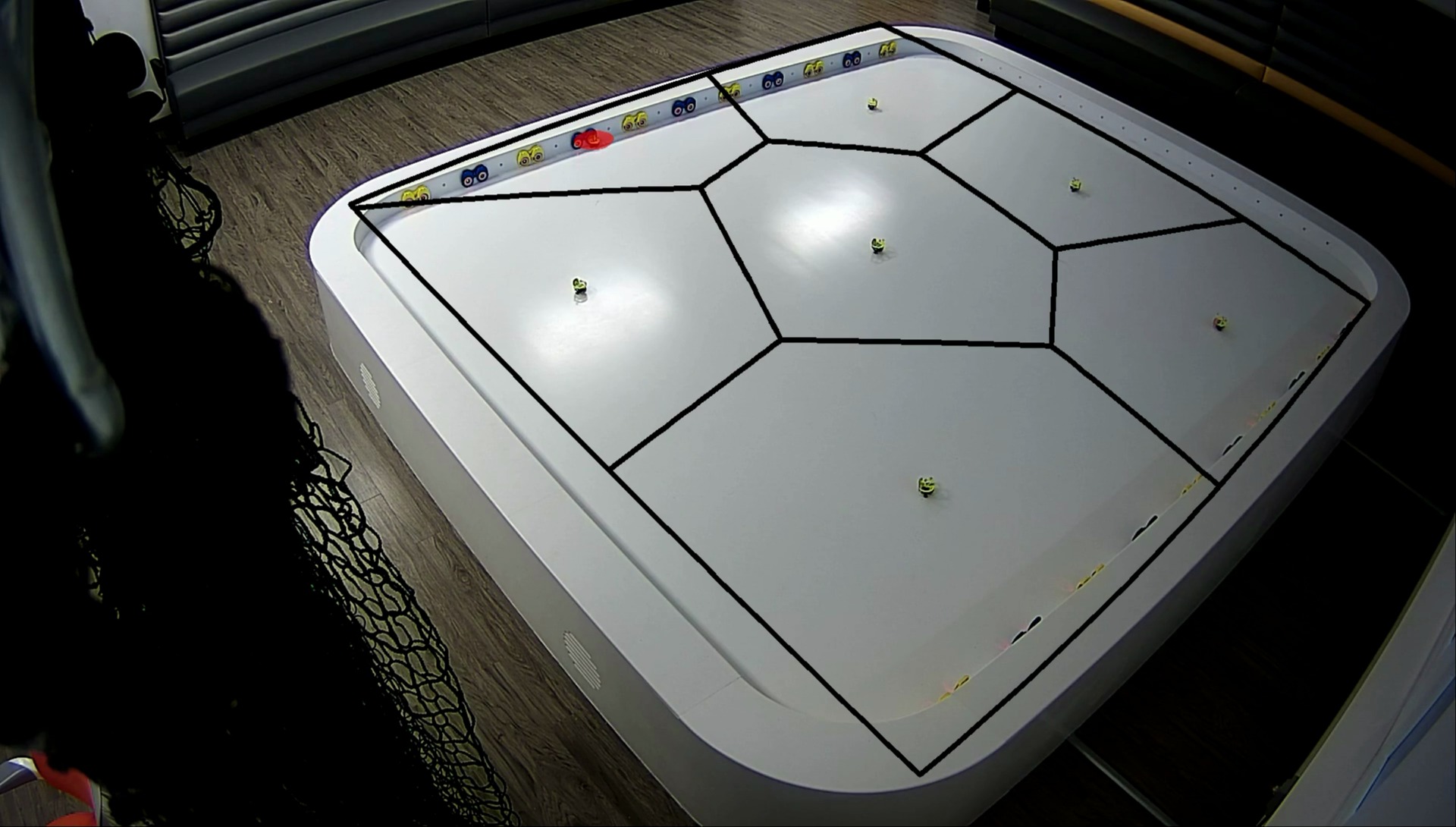}}
		\caption{Sequence of salient frames extracted from the video of the persistent environment surveillance experiment. A team of 7 small differential-drive robots are deployed to perform persistent sensor coverage of the testbed of the Robotarium \cite{pickem2017robotarium}. On one edge of the testbed yellow and blue wireless charging stations are arranged, where the robots can recharge their batteries. The environment (charging) field has been modeled similarly to the one shown in Fig.~\ref{fig:lumpedenergysources}. The black lines represent the boundaries of the Voronoi cells corresponding to each robot. The sequence of images shows the robots performing coverage under the nominal control input (\protect\ref{subfig:bigfig_experiments_1}), two robots, marked with red circles, going back to the charging stations to recharge their batteries (\protect\ref{subfig:bigfig_experiments_2} to \protect\ref{subfig:bigfig_experiments_4}) driven by the controller \eqref{eq:qp}.}
		\label{fig:bigfig_experiments}
	\end{figure}
	Figures \ref{subfig:bigfig_experiments_1} to \ref{subfig:bigfig_experiments_4} show the salient frames of the persistent environment surveillance experiment performed on the Robotarium. The Voronoi cells are superimposed on the frames. The yellow and blue wireless charging station are arranged along one of the edges of the testbed. Following the nominal controller \eqref{eq:coverage}, the robots perform sensor coverage (\ref{subfig:bigfig_experiments_1}). The actual controller executed by the robots is the solution of \eqref{eq:qp}, which allows them to go back and recharge their batteries to prevent the stored energy from going below the minimum desired value $\Emin$ as shown in \ref{subfig:bigfig_experiments_2},~\ref{subfig:bigfig_experiments_3}~and~\ref{subfig:bigfig_experiments_4} (the charging stations occupied by robots are marked in red).
	
	\begin{figure}[th]
		\centering
		\begin{tikzpicture} % scale and rotate
		%	\Large % axis font size
		\begin{axis}
		[
		no marks, % remove marks
		xlabel={$t$ [s]}, % xlabel,
		xtick={0,600,1200,1700}, % xtick
		xticklabels={0,600,1200,1700}, % xticklabel
		ylabel={$\mathcal H$}, % ylabel
		ytick={0,0.1},
		yticklabels={0,0.1},
		xmin=0, % xlim
		xmax=1700, % xlim
		ymin=0, % ylim
		ymax=0.1, % ylim
		enlarge x limits=-1, % x axis tight
		enlarge y limits=-1, % y axis tight
		%		axis equal image, % axis equal
		grid=both, % grid on
		%		xmajorgrids, % x grid on
		%		ymajorgrids, % y grid on
		%		minor tick num=2, % grid minor
		%		legend entries={one-column plot,test12,test34,test56}, % legend
		%		legend style={nodes=right}, % legend style
		%		legend pos= north west, % legend position
		width=0.49\textwidth, % image width
		height=0.25\textwidth % image height
		]
		\addplot [line width=1pt, color=black] table [x=t, y=C, col sep=space]{data/locational_cost.txt};
		\addplot [line width=2pt, color=black] table [x=t, y=C0, col sep=space]{data/locational_cost.txt};
		\end{axis}
		\end{tikzpicture}
		\caption{Locational cost \eqref{eq:loccost} evaluated during the course of a coverage control experiment: the thin line is the cost obtained imposing the energy constraints to the robots, whereas the thick line is the cost obtained assuming infinite availability of energy.}
		\label{fig:c}
	\end{figure}
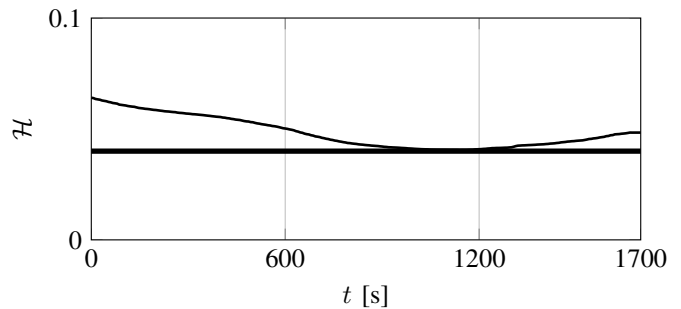
	In Fig.~\ref{fig:c}, the thin line shows the value of the locational cost defined in \eqref{eq:loccost} evaluated during the course of the experiment. Here the information distribution density function $\phi$ has been set to a constant value, meaning that the information is equally spread over the testbed. The thick line of Fig.~\ref{fig:c} represents the value of the cost \eqref{eq:loccost} obtained if no energy constraints were imposed, in which case, the robots are able to asymptotically reach a centroidal Voronoi configuration, where each robot is in the centroid of its corresponding Voronoi cell. This configuration leads to a local minimum of \eqref{eq:loccost}. The value of $\mathcal H$ decreases in correspondence of the situations in which all the available robots are not constrained to charge and are consequently free of following the coverage control input.
	
	\begin{figure}[th]
		\centering
		\begin{tikzpicture} % scale and rotate
		%	\Large % axis font size
		\begin{axis}
		[
		no marks, % remove marks
		xlabel={$t$ [s]}, % xlabel,
		xtick={0,600,1200,1800}, % xtick
		xticklabels={0,600,1200,1800}, % xticklabel
		ylabel={$E$}, % ylabel
		ytick={0,3,3.485000, 4.059000}, % ytick
		yticklabels={0, ,$\Emin$, $\Echg$}, % yticklabel
		xmin=0, % xlim
		xmax=1700, % xlim
		ymin=3.2, % ylim
		ymax=4.2, % ylim
		enlarge x limits=-1, % x axis tight
		enlarge y limits=-1, % y axis tight
		%		axis equal image, % axis equal
		grid=both, % grid on
		%		xmajorgrids, % x grid on
		%		ymajorgrids, % y grid on
		%		minor tick num=2, % grid minor
		%		legend entries={one-column plot,test12,test34,test56}, % legend
		%		legend style={nodes=right}, % legend style
		%		legend pos= north west, % legend position
		width=0.49\textwidth, % image width
		height=0.25\textwidth % image height
		]
		\addplot [line width=1pt, color=clr1] table [x=t, y=b1, col sep=comma]{data/fig1_1new.txt};
		\addplot [line width=1pt, color=clr2] table [x=t, y=b2, col sep=comma]{data/fig1_1new.txt};
		\addplot [line width=1pt, color=clr3] table [x=t, y=b3, col sep=comma]{data/fig1_1new.txt};
		\addplot [line width=1pt, color=clr4] table [x=t, y=b4, col sep=comma]{data/fig1_1new.txt};
		\addplot [line width=1pt, color=clr5] table [x=t, y=b5, col sep=comma]{data/fig1_1new.txt};
		\addplot [line width=1pt, color=clr6] table [x=t, y=b6, col sep=comma]{data/fig1_1new.txt};
		\addplot [line width=1pt, color=clr7] table [x=t, y=b7, col sep=comma]{data/fig1_1new.txt};
		\addplot [line width=2pt, color=black] table [x=t, y=echg, col sep=comma]{data/fig1_1new.txt};
		\addplot [line width=2pt, color=black] table [x=t, y=emin, col sep=comma]{data/fig1_1new.txt};
		\end{axis}
		\end{tikzpicture}
		\caption{Measured battery levels of the 7 robots executing the persistified coverage control experiment: as a result of implementing the input solution of \eqref{eq:qp}, the energy levels of all the robots are constrained to be between the values $\Emin$ and $\Echg$.}
		\label{fig:battxprmnt}
	\end{figure}
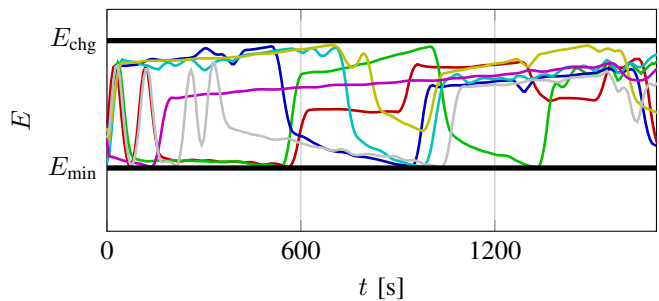
	Figure~\ref{fig:battxprmnt} displays the actual energy stored in the batteries of the robots as measured during the course of the experiment. The inputs to the robots are constrained by \eqref{eq:cbfconstraint} and \eqref{eq:clfconstraint} in such a way that their energies never exceed $\Echg$, while never going below $\Emin$ either. As a results of the optimization-based formulation, each individual robot is able to follow as closely as possible the input encoding the coverage task as long as its battery level satisfies the imposed constraints.
	
	\section{Conclusions}
	\label{sec:conclusions}
	
	In this paper we introduced the concept of robotic task persistification, i.\,e., the process of rendering a robotic task persistent. This allows robots to execute a task over long time horizons, by ensuring that the energy stored in their batteries never gets depleted. The control of the robot energy level is wrapped around the task controller by utilizing control barrier functions and control Lyapunov functions combined in a single optimization-based controller which can be efficiently executed in an online fashion. The result of this formulation is a control framework that allows the robots to execute as closely as possible the assigned task, while simultaneously never depleting the energy in their batteries. This constraint is enforced by expressing the task persistence condition in terms of the forward invariance of a subset of the state space of the robots. The forward invariance property is then ensured by employing control barrier functions, whereas battery recharging is achieved by defining a suitable control Lyapunov function.
	
	The results in Theorem~\ref{thm:nestedcbfs}, Lemma~\ref{thm:tvzcbf}~and~Theorem~\ref{thm:nestedclfs} allow us to apply the presented framework to many different kinds of robots endowed with rechargeable, or even interchangeable, sources of energy. In order to be able to efficiently enforce energy constraints, we insist on the robot dynamic model being in control affine form, case that is often encountered in robotic applications. Since the persistified task is obtained as the output of an optimization problem, the robots are free to execute the given task as closely as possible as long as their source of energy is not discharging below a given lower threshold. The persistification strategy has been applied to environment exploration and monitoring tasks, and it has been tested both in simulation and on an team of ground mobile robots.
	
	\appendices
	\section{}
	\label{app:formulas}
	
	For single integrator robot dynamics and the CBFs \eqref{eq:energycbf} and \eqref{eq:energycbf2}, the expressions of $\frac{\partial h_{i2}}{\partial t}$, $L_f^2 h_{i1}$ and $L_g L_f h_{i1}$ required to evaluate \eqref{eq:cbfconstraint} are given by:
	\begin{equation}
		\begin{split}
			&\begin{aligned}
				\frac{\partial h_{i2}}{\partial t} &= (\Echg+\Emin-2E_i)\frac{\partial F}{\partial w}\frac{\partial w}{\partial I}\frac{\partial I}{\partial t} =\\
				&\!\!\!\!\!\!= (\Echg+\Emin-2E_i) k w^2 \frac{1-E_i}{E_i} \lambda e^{-\lambda(I(x_i,t)-I_\text{c})}\frac{\partial I}{\partial t}
			\end{aligned}\\
			&\begin{aligned}
				L_f^2 h_{i1} &= \bigg( -2F(x_i,E_i,t)+\\
				&+(\Echg+\Emin-2E_i)\frac{\partial F}{\partial E_i} \bigg)F(x_i,E_i,t)
			\end{aligned}\\
			&\begin{aligned}
				L_g &L_f h_{i1} = (\Echg+\Emin-2E_i)\frac{\partial F}{\partial w}\frac{\partial w}{\partial I}\frac{\partial I}{\partial x_i}=\\
				&= (\Echg+\Emin-2E_i) k w^2 \frac{1-E_i}{E_i} \lambda e^{-\lambda(I(x_i,t)-I_\text{c})}\frac{\partial I}{\partial x_i}
			\end{aligned}
		\end{split}
	\end{equation}
	
	For single integrator robot dynamics and the CLF \eqref{eq:chargingclf}, the expressions of $\frac{\partial h_{i2}}{\partial t}$, $L_f^2 V_i$ and $L_g L_f V_i$ required to evaluate \eqref{eq:clfconstraint} are given by:
	
	\begin{equation}
		\begin{split}
			&\begin{aligned}
				\frac{\partial h_{i2}}{\partial t} &= -2(\Echg-E_i)\frac{\partial F}{\partial w}\frac{\partial w}{\partial I}\frac{\partial I}{\partial t}=\\
				&= -2(\Echg-E_i)k w^2 \frac{1-E_i}{E_i} \lambda e^{-\lambda(I(x_i,t)-I_\text{c})}\frac{\partial I}{\partial t},
			\end{aligned}\\
			&L_f^2 V_i = \bigg( 2F(x_i,E_i,t)-2(\Echg-E_i)\frac{\partial F}{\partial E} \bigg)F(x_i,E_i,t),\\
			&\begin{aligned}
				L_g L_f &V_i = -2(\Echg-\Emin)\frac{\partial F}{\partial w}\frac{\partial w}{\partial I}\frac{\partial I}{\partial x_i}=\\
				&= -2(\Echg-\Emin) k w^2 \frac{1-E_i}{E_i} \lambda e^{-\lambda(I(x_i,t)-I_\text{c})}\frac{\partial I}{\partial x_i}.
			\end{aligned}
		\end{split}
	\end{equation}
	
	%\section*{Acknowledgment}
	
	\ifCLASSOPTIONcaptionsoff
	\newpage
	\fi
	
	\bibliographystyle{IEEEtran}
	\bibliography{bib/IEEEabrv,bib/tcst_references}

	\typeout{get arXiv to do 4 passes: Label(s) may have changed. Rerun}
	
\end{document}